\documentclass[pdflatex,sn-mathphys-num]{sn-jnl}

\usepackage{fullpage}
\usepackage{bm}
\usepackage{overpic}

\usepackage{graphicx}%
\usepackage{multirow}%
\usepackage{amsmath,amssymb,amsfonts}%
\usepackage{amsthm}%
\usepackage{mathrsfs}%
\usepackage[title]{appendix}%
\usepackage{xcolor}%
\usepackage{textcomp}%
\usepackage{manyfoot}%
\usepackage{booktabs}%
\usepackage{algorithm}%
\usepackage{algorithmicx}%
\usepackage{algpseudocode}%
\usepackage{listings}%

\theoremstyle{thmstyleone}%
\newtheorem{theorem}{Theorem}[section]
\newtheorem{proposition}{Proposition}[section]

\theoremstyle{thmstyletwo}%
\newtheorem{remark}{Remark}%

\theoremstyle{thmstylethree}%

\newcommand\dd{\mathrm{d}}

\newcommand\x{\bm{x}}

\newcommand\y{\bm{y}}
\newcommand\z{\bm{z}}

\title[Accelerating Particle-based Energetic Variational Inference]{Accelerating Particle-based Energetic Variational Inference}

\author[1]{Xuelian Bao}
\author[2]{Lulu Kang}
\author[3]{Chun Liu}
\author*[4]{Yiwei Wang}\email{yiweiw@ucr.edu}

\affil[1]{\orgdiv{School of Mathematics}, \orgname{South China University of Technology}, \orgaddress{\city{Guangzhou}, \country{China}}}
\affil[2]{\orgdiv{Department of Mathematics and Statistics}, \orgname{University of Massachusetts Amherst}, \orgaddress{\city{Amherst}, \state{MA}, \country{USA}}}
\affil[3]{\orgdiv{Department of Applied Mathematics}, \orgname{Illinois Institute of Technology}, \orgaddress{\city{Chicago}, \state{IL}, \country{USA}}}
\affil[4]{\orgdiv{Department of Mathematics}, \orgname{University of California Riverside}, \orgaddress{\city{Riverside}, \state{CA}, \country{USA}}}

\begin{document}

\abstract{In this work, we propose a new particle-based variational inference (ParVI) method for accelerating the Energetic Variational Inference with Implicit scheme (EVI-Im) introduced in Ref. \cite{wang2021particle}. Inspired by energy quadratization (EQ) and operator splitting techniques for gradient flows,  the proposed method efficiently drives particles towards the target distribution, while retaining a meaningful stability mechanism. Unlike EVI-Im, which employs the implicit Euler method to solve variational-preserving particle dynamics obtained from a ``discretization-then-variation'' approach for minimizing the Kullback--Leibler divergence, the proposed algorithm avoids repeated evaluation of inter-particle interaction terms within each time step, significantly reducing computational cost. The framework is also extensible to other gradient-based sampling techniques. Through several numerical experiments, we demonstrate that the proposed method achieves competitive performance compared with existing ParVI approaches, while offering advantages in efficiency and robustness in certain regimes.}



\maketitle

\section{Introduction}

Many problems in machine learning and modern statistics can be formulated as estimating or sampling from a target distribution $\rho^*(\x)$, which is known up to an intractable normalizing constant.
Two popular classes of methods are Markov Chain Monte Carlo (MCMC) methods \cite{metropolis1953equation, hastings1970monte, geman1984stochastic, welling2011bayesian} and Variational Inference (VI) methods \cite{jordan1999introduction, neal1998view, wainwright2008graphical, blei2017variational}.

The idea of MCMC is to generate samples from the target distribution by constructing a Markov chain whose equilibrium distribution is the target distribution. 
Examples include Metropolis--Hastings algorithm \cite{metropolis1953equation, hastings1970monte}, Gibbs sampling \cite{geman1984stochastic, casella1992explaining}, Langevin Monte Carlo (LMC) \cite{rossky1978brownian, parisi1981correlation, roberts1996exponential, welling2011bayesian}, and Hamiltonian Monte Carlo (HMC) \cite{neal1993probabilistic, duane1987hybrid}. 
In contrast, VI reformulates the inference as an optimization problem and seeks a distribution $\rho$ that is closest to the target distribution $\rho^*$ in terms of a chosen divergence measure:
\begin{equation}\label{op_pro}
\rho^{\rm opt} = \mathop{\arg\min}_{\rho \in \mathcal{Q}} D(\rho || \rho^*).
\end{equation}
Here, $\mathcal{Q}$ is the admissible set that contains all feasible distributions, $D(p || q)$ is a dissimilarity function or a divergence measure that assesses the differences between two probability distributions $p$ and $q$. 
For Bayesian inference problems, $D(p || q)$ is often taken as Kullback--Leibler (KL) divergence \cite{blei2017variational}, given by
\[\text{KL}(\rho || \rho^*) = \int \rho \ln \left(  \frac{\rho}{\rho^*} \right) \dd x =  \int \rho \ln \rho + \rho V(x) \dd x\ ,\]
where $V(x) = - \ln \rho^*$. 
Note that the normalizing constant of $\rho^*$ does not influence the optimal solution and can therefore be omitted in the definition of $V(x)$, which overcomes the biggest challenge in sampling the posterior distribution for Bayesian inference.

\subsection{Particle-based Variational Inference}\label{subsec:parvi}
Many VI methods choose a parametric family of distributions as $\mathcal{Q}$, such as the mean-field approach \cite{blei2017variational} or neural-network-based methods \cite{papamakarios2021normalizing, rezende2015variational}.

However, significant progress has also been made in developing particle-based variational inference (ParVI) methods \cite{Chen2018, chen2019projected, liu2016stein, liu2017stein, liu2018riemannian,  liu2019understanding, maoutsa2020interacting}, such as Stein Variational Gradient Descent (SVGD) method \cite{liu2017stein, liu2016stein}. 
In these ParVI methods, the admissible set  $\mathcal{Q}$ is defined as a set of empirical measures with $N$ samples, or \emph{particles}, i.e., 
$$
\mathcal{Q} = \left\{ \rho_N(\x) ~|~  \rho_N = \frac{1}{N} \sum_{i=1}^N \delta(\x - \x_i) \right\},
$$ 
where $\x_i \in \mathbb{R}^n$ for $i=1,\ldots, N$ are the positions of particles, and $\delta$ is the Dirac delta function, which is a generalized function defined such that \(\int_{\mathbb{R}^n} \delta(\x - \x_i) f(\x) \, d\x = f(\x_i)\) for any smooth test function \(f(\x)\).

ParVI methods can be viewed as non-parametric variational inference approaches, where the goal is to find an optimal set of samples $\{ \x_i \}_{i=1}^N$ that minimizes $D(\rho_N || \rho^*)$ or its approximation, as $D(\rho_N || \rho^*)$ may not be well-defined for an empirical distribution $\rho_N$. 
ParVI methods face a key challenge: how to solve the non-convex optimization problem \eqref{op_pro}, where standard optimization approaches may converge slowly or fail to converge at all. 
To address this challenge, many ParVI methods employ some kind of particle dynamics to iteratively refine the particles' distribution, with the minimization carried out through solving a dynamical evolution equation for the particles $\{\x_i(t) \}_{i=1}^N$:
\begin{equation}\label{particle_v}
\dot{\x}_i(t) = {\bm v}_i (\x_1, \ldots, \x_N; \rho^*),
\end{equation}
where ${\bm v}_i (\x_1, \ldots, \x_N; \rho^*)$ represents the $i$-th particle's velocity, which also depends on the target distribution. 
The velocity often involves interaction terms among particles \cite{maoutsa2020interacting}.

Two main questions still remain in this framework. 
The first one is how to choose the velocity ${\bm v}_i$ appropriately to ensure convergence and efficiency. 
The second question is how to solve the continuous particle dynamics \eqref{particle_v}, which involves selecting a suitable temporal discretization. 
Recently, a framework called Energetic Variational Inference (EVI) was proposed in Ref. \cite{wang2021particle} and it answers both questions. 
This framework is inspired by the energetic variational approach used to model the dynamics of non-equilibrium thermodynamic systems \cite{EVAreview, giga2017variational}. 
The main idea of EVI is to formulate a continuous dynamics of minimizing $D(\rho || \rho^*)$ in the probability space through a continuous energy-dissipation law \cite{giga2017variational}. A similar idea was explored in Refs. \cite{E2020machine, Trillos2020}.
The particle dynamics in EVI is then derived using a ``discretization-then-variation'' approach that consists of two steps. 
In the first step, we discretize the continuous energy-dissipation law. 
In the second step, we apply the energetic variational approach at the particle level to obtain the particles' dynamics.
As an advantage of the ``discretization-then-variation'' approach, the dynamics at the particle level maintains the variational structure, which means the dynamics is a gradient flow in terms of particles and particles evolve in the direction of reducing the discretized energy. 
Applying various numerical techniques for gradient flows, we can obtain the corresponding robust and energy-stable VI algorithms.

As an example of the EVI framework, a new ParVI algorithm, termed EVI-Im, was developed in Ref. \cite{wang2021particle}.
With the KL divergence as the dissimilarity function, the following variational particle dynamics can be derived using EVI:  
\begin{equation}\label{Particle_ODE_1}
\dot{\x}_i = - \left( \frac{ \sum_{j=1}^N  \nabla_{\x_i} K_h(\x_i, \x_j)}{\sum_{j = 1}^N K_h(\x_i, \x_j) }  +  \sum_{k=1}^N \frac{\nabla_{\x_i} K_h(\x_k, \x_i)}{\sum_{j=1}^N K_h(\x_k, \x_j)}+ \nabla_{\x_i} V(\x_i)  \right), \quad i=1, 2, \cdots, N.
\end{equation}

Compared with other KL-divergence-based ParVI, the particle dynamics \eqref{Particle_ODE_1} is an $L^2$-gradient flow of $\{ \x_i \}_{i=1}^N$ associated with the free energy $\mathcal{F}_h$ defined by \cite{carrillo2019blob, reich2021fokker} 
\begin{equation}\label{F_h}
\mathcal{F}_h (\{\x_i\}_{i=1}^N)= \frac{1}N \sum_{i=1}^N \bigg{(} \ln\bigg{(}\frac{1}N  \sum_{j=1}^N K_h(\x_i, \x_j) \bigg{)} +V(\x_i)   \bigg{)}  \ ,
\end{equation}
which is an approximation of the KL-divergence. 
Here, $K_h$ is a regularization kernel with $h$ as the bandwidth tuning parameter. A typical choice of $K_h$ is the Gaussian kernel:
$$
K_h(\x_i, \x_j) = \frac{1}{(\sqrt{2\pi}h)^d}\exp\left(-\frac{|\x_i - \x_j|^2}{2h^2}\right)\ ,
$$
where $d$ is the dimension of the space. Thanks to the variational structure, applying the implicit Euler discretization to \eqref{Particle_ODE_1}, we obtain the optimization problem
\begin{equation}\label{Op_EVI_Im}
 \{ \x_i^{n+1} \}_{i=1}^N = \mathop{\arg\min}_{\{ \x_i \}_{i=1}^N} J_n (\{\x_i\}_{i=1}^N), \quad J_n (\{\x_i\}_{i=1}^N) = \frac{1}{2 \tau N} \sum_{i=1}^N \| \x_i - \x_i^n \|^2 + \mathcal{F}_h ( \{\x_i\}_{i=1}^N )
\end{equation}
at each iteration, where $\tau$ is the time step size for the implicit Euler discretization.
In practice, the optimization problem \eqref{Op_EVI_Im} can be solved by using the gradient descent with Barzilai--Borwein (BB) step size \cite{barzilai1988two} or the AdaGrad method \cite{duchi2011adaptive, Chen2018}. 
The numerical results in Ref. \cite{wang2021particle} demonstrate the strong performance of EVI-Im. 

One feature that distinguishes EVI-Im from other ParVI methods is that EVI-Im results from an $L^2$-gradient flow, which leads to more stable behavior in the sense that the particles are theoretically guaranteed to move towards the target distribution in each update. 
However, most ParVI methods, such as SVGD, are not $L^2$-gradient flows. Thus, the implicit Euler discretization will not lead to an optimization problem. 
In Ref. \cite{wang2021particle}, numerical examples have shown that SVGD is not as stable as EVI-Im. 
On the other hand, the need to solve a nonlinear optimization problem in each iteration is also a bottleneck of EVI-Im. 
In particular, one has to repeatedly evaluate the interaction terms $\nabla_{\x_i} K_h(\x_i,\x_j)$ and $K_h(\x_i,\x_j)$ among particles in \eqref{Particle_ODE_1} when evaluating the gradient term in each iteration, which constitutes a significant computational bottleneck, especially for a large number of particles. 
For these reasons, we focus on EVI-Im method in this article, and our goal is to overcome its computational bottleneck.

\subsection{Our Contribution}
To reduce the computational cost associated with EVI-Im, we propose a new algorithm called {\bf ImEQ (Implicit scheme with partial Energy Quadratization)}, which integrates the energy quadratization technique into gradient flows \cite{yangxf2016, zhaoj2017, shenj2018} to avoid repeatedly calculating the interaction terms in the original EVI-Im. 
This integration yields a new particle-based method. 
Unlike the Adaptive Gradient Descent with Energy (AEGD) method \cite{liu2020aegd}, an optimization method for machine learning problems, which essentially applies the energy quadratization to the entire energy and derives an explicit scheme, the proposed ImEQ method remains implicit and still requires solving an optimization problem in each iteration.
While the computational cost is slightly higher than that of AEGD method, ImEQ exhibits better stability across various examples.

We apply ImEQ method to various particle-based variational inference problems, demonstrating its robustness and efficiency. 
Various numerical results show that ImEQ method can achieve results comparable to EVI-Im while significantly reducing the CPU time when the number of particles is large. 
Additionally, compared to AEGD, the proposed method tends to be more robust in challenging examples, such as when the initial distribution is far from the target, owing to the implicit treatment of the potential part $H$, without introducing significant additional computational cost. We note that AEGD can perform very well when a sufficiently small step size is used, and ImEQ is not intended to uniformly outperform AEGD in every regime.
ImEQ can also be applied to other problems in machine learning, particularly those that can be interpreted as interacting particle systems, such as ensemble Kalman sampler \cite{chen2024bayesian}, particle-based generative models \cite{chen-wangEVI}, consensus-based optimization methods \cite{carrillo2021consensus}, and neural network training \cite{rotskoff2022trainability}.

The rest of this article is organized as follows. 
Section 2 gives a review of the EQ technique for the gradient flow and AEGD method.
In Section 3, we present ImEQ method and apply it to solve the particle ODE system \eqref{Particle_ODE_1}.
In Section 4, we demonstrate the efficiency and robustness of ImEQ method by comparing it with some existing ParVI methods on various synthetic and real-world problems for Bayesian inference. 
Section 5 concludes the article. 

\section{Preliminary: Energy Quadratization (EQ) for gradient flows}\label{sec:pre}

Since many existing optimization methods can be interpreted as temporal discretization of gradient flows, there has been increasing interest in developing effective optimization algorithms based on the continuous formulation of gradient flows \cite{E2020machine, liu2020aegd, zhang2025relaxed}. 
Recently, energetic quadratization approaches, including Invariant Energy Quadratization (IEQ) and Scalar Auxiliary Variable (SAV) methods, have been popular numerical techniques to solve various gradient flow problems \cite{yangxf2016, zhaoj2017, shenj2018}. 
These methods are also used in solving machine learning problems \cite{liu2020aegd, zhang2024energy}. 
For example, in a recent work \cite{liu2020aegd}, the authors adopted this approach and proposed a new optimization method, called {\bf Adaptive Gradient Descent with Energy (AEGD)}. 
Some extensions of AEGD further accelerating the convergence were then proposed in Refs. \cite{liu2022adaptive, liu2022dynamic, liu2023adaptive}.

Classical IEQ and SAV approaches are developed for infinite dimensional gradient flows, with spatial discretization typically applied after temporal discretization \cite{shenj2018, yangxf2016, zhaoj2017}. 
For the finite-dimensional case, the IEQ and SAV methods are essentially the same. 
To explain the ideas behind these energetic quadratization approaches, we consider the following finite-dimensional $L^2$-gradient flow
\begin{equation}\label{ODE_z}
  \dot{\z} = - \nabla F(\z), \quad \z \in \mathbb{R}^d \ ,
\end{equation}
which minimizes the ``free energy'' $F(\z)$ through the energy-dissipation law
\begin{equation}\label{GD_z}
\frac{\dd}{\dd t} F(\z) = - \| \dot{\z} \|^2\ ,
\end{equation}
where $\| \dot{\z} \|^2 =  \dot{\z} \cdot \dot{\z}$ is the standard $l^2$-norm in $\mathbb{R}^d$.
Assume $F(\z)$ is bounded from below. Then one can define $q(\z) = \sqrt{F(\z) + C}$, where $C$ is a pre-specified constant such that $F(\z) + C \geq 0,\ \forall \z \in \mathbb{R}^d$. Numerical tests in previous work show the performance of the resulting IEQ/SAV-type quadratization schemes is often not sensitive to the particular choice of $C$ (as long as $F(\z)+C\ge 0$) \cite{shenj2018, liu2020aegd}.
By the definition and the chain rule, we have $\dot{q}(\z) = \nabla q(\z) \cdot \dot{\z}$ and $\nabla q(\z) = \frac{1}{2 q(\z)} \nabla F(\z)$.
Hence, the gradient flow \eqref{ODE_z} is equivalent to 
\begin{equation}\label{ODE_qz}
    \begin{cases}
      & \dot{\z} = - 2 q(\z) \nabla q(\z), \\
      & \dot{q}(\z) = \nabla q(\z) \cdot \dot{\z}. \\
    \end{cases}
\end{equation}

To solve \eqref{ODE_z} numerically, the classical SAV method introduces a scalar auxiliary variable $r(t) = q(\z(t))$ \cite{shenj2018}. The variables $r(t)$ and $\z(t)$ then satisfy the system of ODEs
\begin{equation}\label{ODE_rqz}
    \begin{cases}
      & \dot{\z} = - 2 r(t) \nabla q(\z), \\
      & \dot{r} = \nabla q(\z)  \cdot \dot{\z}. \\
    \end{cases}
\end{equation}

Next, to obtain a practical algorithm, a semi-implicit temporal discretization is introduced to \eqref{ODE_rqz}: 
\begin{equation}\label{SAV}
  \left\{
  \begin{aligned}
    &  \dfrac{\z^{n+1} - \z^{n}}{\tau} = -  2 r^{n+1}  \nabla q (\z^{n}), \\
    & \dfrac{r^{n+1} - r^{n}}{\tau} = \nabla q (\z^{n}) \cdot (-  2 r^{n+1}   \nabla q (\z^{n}) ),\\
  \end{aligned}
\right. 
\end{equation}
where $\tau$ is the step size.
Although the numerical scheme looks complicated, it can be rewritten as a fully explicit scheme \cite{liu2020aegd}. Precisely, one can first solve $r^{n+1}$ using the second equation, which gives
\begin{equation}\label{r_n_1_form}
 r^{n+1} = \dfrac{r^n}{1 + 2 \tau  \|  \nabla q (\z^{n}) \|^2 }.
\end{equation}
Recall the previous definition of $q(\z)$ and the resulting formula of $\nabla q(\z)$.
With $r^{n+1}$, we can update $\z^{n+1}$ explicitly by
\begin{equation}\label{SAVz}
    \z^{n+1} = \z^{n} -  \tau (2 r^{n+1}  \nabla q (\z^{n})) = \z^{n} -   \tau  \frac{r^{n+1}}{\sqrt{ F(\z^n) + C}}  \nabla F(\z^n).
\end{equation}
The ratio $r^{n+1}/\sqrt{ F(\z^n) + C}$ can be interpreted as a scaling factor for the step size of the gradient descent dynamics. Therefore, this method was named as an adaptive gradient descent with energy (AEGD) in Ref. \cite{liu2020aegd}. 

An advantage of the scheme \eqref{SAVz} is that it is an explicit scheme but can achieve certain stability in the sense of a modified energy $\tilde{F}(r, \z) = r^2$  (see Proposition \ref{prop1}).
Here, $\tilde{F}(r, \z) = r^2$ represents a quadratic transformation of $F(\z)$, where $r$ serves as an auxiliary variable. Although $\tilde{F}$ is independent of $\z$ explicitly in this case, we use the notation $\tilde{F}(r, \z)$ to emphasize that $\z$ is the original variable. This technique is named ``energy quadratization'' as it transfers a general nonlinear function of $\z$ to a quadratic function of $r$ \cite{zhaoj2017}.

\begin{proposition}\label{prop1}
The numerical scheme \eqref{SAV} satisfies the following energy stability:
\begin{equation}
\tilde{F}^{n+1} - \tilde{F}^{n}
\leq - \frac{ 1 }{\tau} \|\z^{n+1} - \z^n \|^2\ ,
\end{equation}
with $\tilde{F}^n = (r^{n})^2$.
\end{proposition}

\begin{proof}
Multiplying the second equation in  \eqref{SAV} by $2 r^{n+1}$, and using the first equation in  \eqref{SAV}, we have
\begin{equation}
 2 (r^{n+1} - r^n) r^{n+1} = - \frac{1}{\tau} \| \z^{n+1} - \z^n \|^2.
\end{equation}
Using the identity $2 a^2 - 2 ab  = a^2 - b^2 + (a-b)^2$, we have 
\begin{equation}\label{Decay_r}
 (r^{n+1})^2 - (r^n)^2 = - \frac{ 1 }{\tau} \|\z^{n+1} - \z^n \|^2 -  (r^{n+1} - r^n)^2   \leq - \frac{ 1 }{\tau} \|\z^{n+1} - \z^n \|^2.
\end{equation}
\end{proof}

Since $\tilde{F}(r, \z) = r^2$ can be viewed as a certain approximation of the energy $F(\z)$, Eq. \eqref{Decay_r} provides a form of stability for the scheme. However, it is worth mentioning that the temporal discretization \eqref{SAV} cannot guarantee $r^{n+1} = q(\z^{n+1})$. 
Consequently, we do not have the energy stability in terms of the original energy $F(\z^{n+1}) \leq F(\z^n)$. 
We may need to choose a very small $\tau$ value such that the algorithm works \cite{shen2018convergence}.

\begin{remark}\label{remark:lag}
Here, we provide a more intuitive derivation or interpretation of system \eqref{ODE_rqz} through the Lagrange multiplier approach. 
The original gradient flow \eqref{GD_z} can be reformulated as a constrained gradient flow 
\begin{equation}\label{lagmulti}
  \frac{\dd}{\dd t} \tilde{F}(r, \z) = - \| \dot{\z} \|^2, \quad {\rm with \  constraint} \quad    r = q(\z).
\end{equation}
The constraint $r = q(\z)$ ensures that the energies $\tilde{F}(r, \z)$ and $F(\z)$ are essentially equivalent. 
We define the Lagrangian function 
\begin{equation}
    \mathcal{F}(r, \z; \lambda) = \tilde{F}(r, \z)- \lambda ( r - q(\z)) = r^2 - \lambda ( r - q(\z))\ ,
\end{equation}
where $\lambda$ is the Lagrange multiplier.
According to the method of Lagrange multipliers, the minimizer of $\tilde{F}(r, \z)$ under the constraint $r = q(\z)$ satisfies
\[
 \frac{\partial \mathcal{F}}{\partial r} =0, \quad  \frac{\partial \mathcal{F}}{\partial \z} =0, \text{ and } \frac{\partial \mathcal{F}}{\partial \lambda}=0.
\]
The first equation leads to 
\begin{equation}
   \frac{\partial \mathcal{F}}{\partial r} =  2r  - \lambda = 0 \quad \Rightarrow \quad \lambda  = 2r\ .
\end{equation}
The third equation is simply the constraint $r=q(\z)$. 
The second equation is $\partial \mathcal{F}/\partial \z=\lambda \nabla q(\z)=0$, which can be solved by the following equivalent gradient flow 
\begin{equation}\label{zgradientflow}
\dot{\z} = - \frac{\partial \mathcal{F}(r, \z; \lambda) }{\partial \z} = - \lambda \nabla q (\z) = - 2 r \nabla q(\z).
\end{equation}
To find the equation of $r(t)$, we take time-derivative of the constraint, which gives
\begin{equation}\label{zgradientflow2}
\dot{r} = \nabla q(\z) \cdot \dot{\z}.
\end{equation}
Eqs. \eqref{zgradientflow} and \eqref{zgradientflow2} are exactly the same as the system \eqref{ODE_rqz}. 
However, this alternative derivation of \eqref{ODE_rqz} does not change the fact that the constraint $r = q(\z)$ cannot be satisfied after the temporal discretization, which is also remarked above in Proposition \ref{prop1}. 
\end{remark}

Since the energetic quadratization technique or AEGD method is developed for general gradient flows, it is straightforward to apply it to the particle dynamics \eqref{Particle_ODE_1}. 
However, as shown in Section \ref{sec:examples}, the standard AEGD may fail to explore the target distribution efficiently and lead to an unsatisfactory result in several Bayesian inference problems, particularly when the initial distribution is far away from the target distribution (see Fig. \ref{star_500_1} for an example).
In this paper, we propose a new algorithm, which only applies energetic quadratization to some part of the free energy instead of its entirety.

\section{Accelerating EVI-Im via partial energy quadratization}

In this section, we introduce ImEQ method. The main idea is to apply energy quadratization specifically to the computationally expensive or non-convex components of the free energy, while maintaining an implicit treatment for the potential/convex parts to ensure robustness. By applying the ImEQ framework to the dynamic particle ODE system \eqref{Particle_ODE_1}, we derive a highly efficient algorithm for particle-based variational inference. This method significantly reduces the computational cost associated with the fully implicit EVI-Im scheme while preserving its stability advantages.

\subsection{ImEQ Discretization}
Consider a finite dimensional $L^2$-gradient flow
\begin{equation}\label{ODE_Fz}
  \dot{\z} = - \nabla F(\z), \quad \z \in \mathbb{R}^d, 
\end{equation}
where $F(\z)$ is the objective function. Assume that $F(\z)$ can be decomposed into two parts, i.e.,
\begin{equation}\label{Decomp}
    F(\z) = G(\z) + H(\z),
\end{equation}
where $G(\z)$ is a function that is bounded from below. 
Such a decomposition is certainly non-unique. 
To define an ImEQ algorithm, we only require $G$ to be bounded from below.
Moreover, to obtain existence/uniqueness of the implicit update and the discrete dissipation of the modified energy, we impose a convexity condition on $H$ in our analysis. 
From a computational perspective, we choose $G$ as the remaining component whose gradient is expensive to evaluate (e.g., interaction/entropy-surrogate terms).

Instead of introducing an energy quadratization to $F(\z)$ as in Ref. \cite{liu2020aegd}, we can only introduce an energy quadratization to $G(\z)$, i.e., let $q(\z) = \sqrt{G(\z) + C}$. 
Similar to the classical SAV method, we introduce an auxiliary variable $r(t) = q(\z(t))$. 
Following the derivation of the Lagrange multiplier approach in Remark \ref{remark:lag}, we reformulate the gradient flow as a constrained gradient flow
\begin{equation}
  \frac{\dd}{\dd t} \tilde{F}(\z, r) = - |\dot{\z}|^2, \quad {\rm with~constraint~}~  r = q(\z), 
\end{equation}
where $\tilde{F}(\z, r) = r^2 + H(\z)$ is the modified energy. 
As in Remark \ref{remark:lag}, we define a Lagrange multiplier and construct the following augmented functional:
\[
\hat{\mathcal{F}}(r, \z; \lambda) = \tilde{F}(\z, r)- \lambda (r - q(\z)) = r^2 + H(\z) - \lambda (r - q(\z)) ,
\]
where $\lambda$ is the Lagrange multiplier. 
Following the same derivation in Remark \ref{remark:lag}, that $r$ reaches equilibrium instantaneously, we obtain
\[
 2 r - \lambda = 0. 
\]
In the meantime, by taking the variation of $\z$, we have
\[
\dot{\z} = - (\nabla H(\z) + \lambda \nabla q(\z)) = - \nabla H(\z) - 2 r \nabla q(\z).
\]
In order to close the system, we take the time derivative of both sides of the constraint and obtain $\dot{r} = \nabla q(\z) \cdot \dot{\z}$.  
Therefore, the final system reads as follows,
\begin{equation}\label{ODE_ImEQ}
    \begin{cases}
      & \dot{\z} = - \nabla H(\z) - 2 r(t) \nabla q(\z), \\
      & \dot{r} = \nabla q(\z) \cdot  \dot{\z}. 
    \end{cases}
\end{equation}
A practical numerical scheme can be obtained by introducing a temporal discretization to \eqref{ODE_ImEQ} using a semi-implicit scheme: 
\begin{equation}\label{ImEQscheme}
\left\{
    \begin{aligned}
    & \frac{\z^{n+1} - \z^n}{\tau} = - \nabla H(\z^{n+1}) - 2 r^{n+1} \nabla q(\z^n) ,\\
    & \frac{r^{n+1} - r^n}{\tau} = \nabla q(\z^n) \cdot \frac{\z^{n+1} - \z^n}{\tau}.\\
    \end{aligned}
    \right.
\end{equation}
Although we no longer have an explicit update rule as in AEGD method \eqref{SAV}, the coupled system \eqref{ImEQscheme} can still be solved efficiently. 
Rewrite the second equation in \eqref{ImEQscheme}, and we obtain
\begin{equation}\label{eq_qq1}
r^{n+1} = r^n + \nabla q(\z^n) \cdot (\z^{n+1} - \z^n).
\end{equation}
Substituting (\ref{eq_qq1}) into the first equation of \eqref{ImEQscheme} and rearranging the terms, we have
\begin{equation}\label{eq_qq2}
    \frac{1}{\tau} \big[ {\sf I}  + 2 \tau  \nabla q (\z^n) \otimes  \nabla  q ( \z^n )  \big] (\z^{n+1} - \z^n) = - 2 r^n  \nabla  q(\z^n) - \nabla  H(\z^{n+1}),
\end{equation}
where ${\sf I}$ is the identity matrix and $\x \otimes \x=\x \x^\top$. Let ${\sf B}^n = {\sf I} +  2 \tau   \nabla q (\z^n) \otimes  \nabla  q ( \z^n )$. It is straightforward to show that ${\sf B}^n$ is positive-definite. The nonlinear equation (\ref{eq_qq2}) can be reformulated as an optimization problem
\begin{equation}\label{op_qq2}
\z^{n+1} = \mathop{\arg\min}_{\z} \tilde{J_n}(\z), \quad \tilde{J_n}(\z) = \bigg( \frac{1}{2 \tau} \| \z - \z^{n} \|^2_{{\sf B}^n} + H(\z) + (2 r^n  \nabla  q ( \z^n ), \z - \z^n)\bigg),  
\end{equation}
where $\| \z \|^2_{{\sf B}^n} =  \z^{\top} {{\sf B}^n} \z$ is the weighted $l^2$-norm. Indeed, the first-order optimality condition for \eqref{op_qq2} is
\[
\nabla \tilde{J}_n(\z)
= \frac{1}{\tau}{\sf B}^n(\z-\z^n) + \nabla H(\z) + 2r^n\nabla q(\z^n)=0.
\]
Evaluating this condition at $\z=\z^{n+1}$ gives \eqref{eq_qq2}.
Conversely, any solution of \eqref{eq_qq2} is a critical point of $\tilde{J}_n$. 
The proposed numerical scheme \eqref{ImEQscheme} is an implicit algorithm, combined with partial energy quadratization. Therefore, we refer to it as ImEQ method. 

\begin{remark}
The scheme \eqref{eq_qq2} reduces to AEGD scheme \eqref{SAV} if $H(\z) = 0$. Indeed, notice that
\begin{equation}
( {\sf I}  + 2 \tau  \nabla q (\z^n) \otimes  \nabla  q ( \z^n )  )^{-1} =  {\sf I} - \frac{ 2 \tau 
(\nabla q (\z^n)) (\nabla q (\z^n))^\top }{ 1 + 2 \tau |\nabla q (\z^n)|^2 }. 
\end{equation}
Substituting into \eqref{eq_qq2}, we have
\begin{equation}
\begin{aligned}
\frac{1}{\tau} 
(\z^{n+1} - \z^n) & = - \frac{2 r^n}{1 + 2 \tau |\nabla q (\z^n)|^2} \nabla q(\z^n) - \left( {\sf I} - \frac{ 2 \tau ( \nabla q (\z^n)) (\nabla q (\z^n))^{\rm T} 
}{ 1 + 2 \tau   |\nabla q (\z^n)|^2 }  \right) \nabla H(\z^{n+1}), 
\end{aligned}
\end{equation}
which reduces to \eqref{SAVz} since $H(\z) = 0$ and (\ref{r_n_1_form}) holds.
\end{remark}

\begin{theorem}\label{thm:unique_solvability}
Assume $H$ is twice differentiable:
\begin{itemize}
\item \textbf{Case 1} If $H$ is convex, then for any $\tau>0$ the objective $\tilde{J}_n$ in~\eqref{op_qq2} is strongly convex. Hence~\eqref{op_qq2} admits a unique minimizer $\z^{n+1}\in\mathbb{R}^d$, equivalently~\eqref{eq_qq2} has a unique solution.
\item \textbf{Case 2} If $H$ satisfies $\nabla^2 H(\z)\succeq -L\,{\sf I}$ for all $\z\in\mathbb{R}^d$,
then $\tilde{J}_n$ is strongly convex and admits a unique minimizer provided that $\tau < \frac{1}{L}.$
\end{itemize}
\end{theorem}

\begin{proof}
The Hessian of $\tilde{J}_n$ is
\[
\nabla^2 \tilde{J}_n(\z)=\frac{1}{\tau}{\sf B}^n + \nabla^2 H(\z).
\]
Since ${\sf B}^n \succeq {\sf I}$, we have $\frac1{\tau}{\sf B}^n \succeq \frac1{\tau}{\sf I}$.

\smallskip
\noindent
\textbf{Case 1.} If $H$ is convex, then $\nabla^2 H(\z)\succeq 0$ and thus
\[
\nabla^2 \tilde{J}_n(\z)\succeq \frac{1}{\tau}{\sf B}^n \succ 0,
\]
so $\tilde{J}_n$ is strongly convex and has a unique minimizer.

\smallskip
\noindent
\textbf{Case 2.} If $\nabla^2 H(\z)\succeq -L\,{\sf I}$ holds, then
\[
\nabla^2 \tilde{J}_n(\z)\succeq \frac{1}{\tau}{\sf I} - L{\sf I}
= \Big(\frac{1}{\tau}-L\Big){\sf I}.
\]
Therefore, when $\tau<1/L$, we have $\nabla^2\tilde{J}_n(\z)\succ 0$ for all $\z$, and $\tilde{J}_n$ is strongly convex, hence admits a unique minimizer. 
\end{proof}

\begin{theorem}\label{thm:stability}
Assume $H$ is convex and twice differentiable. Let $\{(\z^n, r^n)\}$ be the sequence generated by ImEQ scheme \eqref{ImEQscheme}. For any time step $\tau > 0$, the scheme is unconditionally stable with respect to the modified energy $\tilde{F}(\z, r) = r^2 + H(\z)$, in the sense that
\begin{equation}\label{ImEQenergydecay}
    \tilde{F}(\z^{n+1}, r^{n+1}) \leq \tilde{F}(\z^n, r^n).
\end{equation}
Specifically, the following energy dissipation law holds:
\begin{equation}\label{dissipation_law}
    \tilde{F}(\z^{n+1}, r^{n+1}) - \tilde{F}(\z^n, r^n) = - (r^{n+1} - r^n)^2 - \frac{1}{\tau} \|\z^{n+1} - \z^n\|^2 - \mathcal{D}_H(\z^{n+1}, \z^n) \leq 0,
\end{equation}
where $\mathcal{D}_H(\z^{n+1}, \z^n) = \frac{1}{2} (\z^{n+1} - \z^n)^\top \nabla^2 H(\xi) (\z^{n+1} - \z^n) \geq 0$ for some $\xi$ between $\z^n$ and $\z^{n+1}$.
\end{theorem}

\begin{proof}
Multiplying the first equation of \eqref{ImEQscheme} by $(\z^{n+1} - \z^n)$, we obtain:
\begin{equation}\label{pf1}
    \frac{1}{\tau} \|\z^{n+1} - \z^n\|^2 = - \nabla H(\z^{n+1}) \cdot (\z^{n+1} - \z^n) - 2r^{n+1} \nabla q(\z^n) \cdot (\z^{n+1} - \z^n).
\end{equation}
From the second equation of \eqref{ImEQscheme}, we have $\nabla q(\z^n) \cdot (\z^{n+1} - \z^n) = r^{n+1} - r^n$. Substituting this into \eqref{pf1} yields:
\begin{equation}\label{pf2}
    \frac{1}{\tau} \|\z^{n+1} - \z^n\|^2 = - \nabla H(\z^{n+1}) \cdot (\z^{n+1} - \z^n) - 2r^{n+1} (r^{n+1} - r^n).
\end{equation}
Using the identity $2b(b-a) = b^2 - a^2 + (b-a)^2$, we have $2r^{n+1}(r^{n+1}-r^n) = (r^{n+1})^2 - (r^n)^2 + (r^{n+1}-r^n)^2$. 
Meanwhile, the Taylor expansion of $H(\z)$ at $\z^{n+1}$ gives:
\begin{equation*}
    H(\z^n) = H(\z^{n+1}) + \nabla H(\z^{n+1}) \cdot (\z^n - \z^{n+1}) + \frac{1}{2} (\z^n - \z^{n+1})^\top \nabla^2 H(\xi) (\z^n - \z^{n+1}),
\end{equation*}
which implies $-\nabla H(\z^{n+1}) \cdot (\z^{n+1} - \z^n) = H(\z^{n+1}) - H(\z^n) + \mathcal{D}_H(\z^{n+1}, \z^n)$. 
Substituting these into \eqref{pf2}, we arrive at:
\begin{equation*}
    \frac{1}{\tau} \|\z^{n+1} - \z^n\|^2 = H(\z^{n+1}) - H(\z^n) + \mathcal{D}_H(\z^{n+1}, \z^n) - \left( (r^{n+1})^2 - (r^n)^2 + (r^{n+1}-r^n)^2 \right).
\end{equation*}
Rearranging the terms leads to the dissipation law \eqref{dissipation_law}. Since $H$ is convex, $\nabla^2 H \succeq 0$, all terms on the right side of the rearranged equation are non-positive, thus $\tilde{F}(\z^{n+1}, r^{n+1}) - \tilde{F}(\z^n, r^n) \leq 0$.
\end{proof}

As in AEGD/SAV-type schemes, the pointwise constraint $r^n=\sqrt{G(\z^n)+C}$ is not enforced during the time-discrete update and may deviate at finite step sizes. In practice, a small step $\tau$ is required such that the mismatch is controlled, so the discrete dynamics stay close to the interacting-particle gradient flow.

More precisely, let
\(
\delta^n:=r^n-q(\z^n)\ .
\)
From~\eqref{ImEQscheme} we have
\(
r^{n+1}=r^n+\nabla q(\z^n)\cdot(\z^{n+1}-\z^n).
\)
On the other hand, by a Taylor expansion, we have
\[
q(\z^{n+1})=q(\z^n)+\nabla q(\z^n)\cdot(\z^{n+1}-\z^n)
+\frac12 (\z^{n+1}-\z^n)^{\top}\nabla^2 q(\xi^n)(\z^{n+1}-\z^n),
\]
for some $\xi^n$ on the segment between $\z^n$ and $\z^{n+1}$.
Subtracting the two relations gives
\[
\delta^{n+1}=\delta^n-\frac12 (\z^{n+1}-\z^n)^{\top}\nabla^2 q(\xi^n)(\z^{n+1}-\z^n).
\]
Thus, if the iterates stay in a bounded region where $\nabla^2 q$ is bounded and $\|\z^{n+1}-\z^n\|=O(\tau)$, the per-step change of $\delta^n$ is $O(\tau^2)$, and over a finite time horizon this suggests $|\delta^n|=O(\tau)$.

The above control also clarifies the dynamic consistency of ImEQ.
Using $2q(\z)\nabla q(\z)=\nabla G(\z)$ and the identity
\(
r^{n+1}-q(\z^n) = \delta^{n+1}
+\bigl(q(\z^{n+1})-q(\z^n)\bigr),
\)
the $\z$-update in~\eqref{ImEQscheme} can be viewed as a perturbation of the standard first-order IMEX (semi-implicit) discretization
\[
\frac{\z^{n+1}-\z^n}{\tau}=-\nabla H(\z^{n+1})-\nabla G(\z^n).
\]
Specifically, the deviation from IMEX consists of (i) a term proportional to the constraint mismatch $\delta^{n+1}$ and
(ii) a freezing term proportional to $q(\z^{n+1})-q(\z^n)$.
Under boundedness of $\nabla q$ and in the small-step regime where $|\delta^{n+1}|=O(\tau)$ and $|q(\z^{n+1})-q(\z^n)|=O(\tau)$, the overall perturbation is $O(\tau)$, and therefore ImEQ is first-order consistent with the underlying gradient-flow dynamics as $\tau\to0$.

The modified-energy stability in Theorem~\ref{thm:stability} controls $\tilde F(\z,r)=H(\z)+r^2$, whereas the original objective is
$F(\z)=H(\z)+G(\z)=H(\z)+q(\z)^2-C$.
Their difference satisfies
\[
\tilde F(\z^n,r^n)-(F(\z^n)+C)=(r^n)^2-q(\z^n)^2=(r^n-q(\z^n))(r^n+q(\z^n)).
\]
Hence, if the mismatch $|r^n-q(\z^n)|$ remains small and $r^n,q(\z^n)$ remain bounded along the iterates, then $\tilde F(\z^n,r^n)$ and $F(\z^n)+C$ stay close.
In particular, in the small-step regime where $|r^n-q(\z^n)|=O(\tau)$, monotone decay of $\tilde F$ implies that $F(\z^n)$ is \emph{approximately} decreasing up to a controlled $O(\tau)$ discrepancy.
This provides an explanation for why ImEQ behaves similarly to energy-stable implicit schemes when $\tau$ is sufficiently small, while for large $\tau$ the mismatch can grow and the discrete dynamics may deviate from the intended objective.

A similar mismatch-based viewpoint also applies to AEGD. The key practical difference is that ImEQ treats the convex potential part $H$ implicitly, which helps control the increments $\|\z^{n+1}-\z^n\|$ and thereby keeps the mismatch $\delta^n$ better behaved. In contrast, for AEGD the corresponding updates are typically more explicit, so $\|\z^{n+1}-\z^n\|$ can be harder to control for the same step size, leading to a more restrictive small-step regime in practice. A fully quantitative characterization of such step-size conditions in terms of $H$ and $G$ is beyond the scope of this work and will be pursued in future analysis. In the numerical section, we provide evidence that, in the Double-banana example, ImEQ remains stable for step sizes up to about $\tau\le 10^{-2}$, whereas AEGD typically requires a smaller step size on the order of $\tau\le 10^{-3}$ to exhibit comparable stability and convergence behavior.

\begin{theorem}
Under the same assumptions as Theorem \ref{thm:stability}, if we further assume that $H(\mathbf{z})$ is bounded from below, the sequence $\{(\mathbf{z}^n, r^n)\}$ generated by ImEQ scheme satisfies:
\begin{enumerate}
    \item The modified energy sequence $\{\tilde{F}(\mathbf{z}^n, r^n)\}_{n=0}^\infty$ is monotonically decreasing and converges to a finite value $F^*$.
    \item The successive difference of the iterates vanishes:
    \begin{equation}
        \lim_{n \to \infty} \|\mathbf{z}^{n+1} - \mathbf{z}^n\| = 0.
    \end{equation}
\end{enumerate}
\end{theorem}

\begin{proof}
According to the result proved in Theorem \ref{thm:stability}, $\tilde{F}(\mathbf{z}^n, r^n)$ is a monotonically decreasing sequence. Since $H(\mathbf{z})$ and $G(\mathbf{z})$ are bounded below, $\tilde{F}$ is bounded below. By the Monotone Convergence Theorem, $\tilde{F}(\mathbf{z}^n, r^n) \to F^*$ as $n \to \infty$.

From the dissipation inequality, we can derive:
\begin{equation*}
    \frac{1}{\tau} \|\mathbf{z}^{n+1} - \mathbf{z}^n\|^2 \leq \tilde{F}(\mathbf{z}^n, r^n) - \tilde{F}(\mathbf{z}^{n+1}, r^{n+1}).
\end{equation*}
Summing this inequality from $n=0$ to $N$:
\begin{equation*}
    \frac{1}{\tau} \sum_{n=0}^{N} \|\mathbf{z}^{n+1} - \mathbf{z}^n\|^2 \leq \tilde{F}(\mathbf{z}^0, r^0) - \tilde{F}(\mathbf{z}^{N+1}, r^{N+1}) \leq \tilde{F}(\mathbf{z}^0, r^0) - F^*.
\end{equation*}
As $N \to \infty$, the series on the left side converges, which implies that the general term must go to zero:
\begin{equation*}
    \lim_{n \to \infty} \|\mathbf{z}^{n+1} - \mathbf{z}^n\|^2 = 0.
\end{equation*}
This concludes the proof.
\end{proof}

\begin{remark}
As in AEGD method, the constraint $r^n = \sqrt{G(\z^{n}) + C}$ no longer holds due to temporal discretization. 
   Consequently, we need to set $\tau$ to be sufficiently small to make the algorithm converge in practice. 
   Numerical tests in the next section show that for ImEQ, the step size needs to be smaller than that for EVI-Im. 
   However, the step size in ImEQ can still be larger than that used in AEGD, meaning it can converge faster than AEGD. 
\end{remark}

\subsection{Application to the Interacting Particle System \eqref{Particle_ODE_1}}

Next, we apply the proposed ImEQ method to the interacting particle system \eqref{Particle_ODE_1}. We first decompose $\mathcal{F}_h (\{\x_i\}_{i=1}^N)$ by setting $G(\{\x_i\}_{i=1}^N)$ as the interaction part and $H(\{\x_i\}_{i=1}^N)$ as the potential part, i.e., 
\begin{equation}\label{decomp_ImEQ}
G = \frac{1}N \sum_{i=1}^N \ln\bigg{(}\frac{1}N  \sum_{j=1}^N K_h(\x_i, \x_j) \bigg{)}, \  \mbox{and }  H = \frac{1}N \sum_{i=1}^N V(\x_i). 
\end{equation}
Since the Gaussian kernel $K_h$ is positive and bounded, the interaction term $G(\mathbf{z})$ is bounded from below, ensuring that the auxiliary variable $q = \sqrt{G(\mathbf{z}) + C}$ is well-defined for a sufficiently large constant $C$. Moreover, $H$ is convex provided that $V$ is convex. This convexity assumption is natural from the viewpoint of the mean-field limit of \eqref{Particle_ODE_1}, which leads to the Fokker--Planck equation $\rho_t = \nabla \cdot ( \nabla \rho + \rho \nabla V)$. The convergence-to-equilibrium property of the Fokker--Planck equation often relies on a convexity assumption on $V$. Specifically, the Bakry--\'Emery curvature condition $\nabla^2 V \succeq \lambda \mathbf{I}$ (for $\lambda > 0$) implies the Logarithmic Sobolev Inequality (LSI), which guarantees exponential convergence to equilibrium in Kullback--Leibler divergence \cite{markowich2000trend}.

The decomposition \eqref{decomp_ImEQ} also often improves computational efficiency.
As mentioned in Section \ref{subsec:parvi}, the computational bottleneck of solving the particle system \eqref{Particle_ODE_1} via the implicit Euler (EVI-Im) lies in that we have to evaluate the interaction terms
\begin{equation}\label{ineraction}
\frac{ \sum_{j=1}^N  \nabla_{\x_i} K_h(\x_i, \x_j)}{\sum_{j = 1}^N K_h(\x_i, \x_j) }   +  \sum_{k=1}^N \frac{\nabla_{\x_i} K_h(\x_k, \x_i)}{\sum_{j=1}^N K_h(\x_k, \x_j)}
\end{equation}
frequently in solving the optimization problem \eqref{Op_EVI_Im}. 
In ImEQ, we only need to evaluate the interaction terms once to compute $r^n$ at each time step. 
Moreover, the optimization problem (\ref{op_qq2}) is often easier to solve than that of EVI-Im. 
Consequently, for the same temporal step size, ImEQ can be more than $K$ times faster, where $K$ is the maximum number of iterations inside the optimization procedure at each time step (see Table~\ref{comparediffN} for numerical evidence). 
It is worth pointing out that the step-size in EVI-Im can be significantly larger than that in ImEQ according to the numerical tests, as EVI-Im is unconditionally energy-stable while ImEQ is energy stable only in the sense of the modified energy $\tilde{\mathcal{F}}(\z, r) = r^2 + H(\z)$. For instance, in the Double-banana example in Fig.~\ref{doublebanana}(a), we report results for EVI-Im with $\tau=0.1$ and for ImEQ with $\tau=0.01$.

\begin{remark}
Similar partial energy quadratization approaches have been used in the original IEQ and SAV methods \cite{yangxf2016, shenj2018} for infinite dimensional gradient flows. These methods often consider a free energy
\begin{equation}
\mathcal{F}[\phi] = \int_{\Omega} \phi \mathcal{L} \phi + F(\phi) \dd \x,
\end{equation}
where $\phi(\x, t)$ is an unknown function, $\mathcal{L}$ is a linear, self-adjoint, positive-definite operator, $F(\phi)$ is a nonlinear function of $\phi$, and the  energy quadratization is only applied to $F(\phi)$. We emphasize that for our free energy $\mathcal{F}_h (\{ \x_i \}_{i=1}^N)$, there exist no linear parts. So it is important to choose a suitable decomposition.
Moreover,  different from the most classical IEQ/SAV methods that focus on developing a linear scheme for the underlying PDE, we apply the EQ technique to the interaction terms, aiming to reduce the computational cost. 

\end{remark}

\section{Numerical Examples}\label{sec:examples}
In this section, we demonstrate the effectiveness and robustness of the proposed ImEQ method for ParVI on various synthetic and real-world problems.  We will compare the proposed method with the following existing methods:
\begin{itemize}
\item Blob method \cite{Chen2018}, which minimizes the free energy \eqref{F_h} using the AdaGrad method.
\item EVI-Im method \cite{wang2021particle}, which uses implicit Euler to solve  \eqref{Particle_ODE_1}.
\item AEGD method \cite{liu2020aegd}.
\end{itemize}
Additionally, we use the Stein Variational Gradient Descent (SVGD) method \cite{wang2019stein}, the most widely used ParVI algorithm, as a benchmark for Bayesian logistic regression and Bayesian neural network problems. 

In kernel-based particle variational inference methods, the choice of the kernel bandwidth $h$ plays a critical role: it controls the locality of particle interactions and directly affects how accurately the induced interaction term approximates the underlying score function ($\nabla \ln \rho$). 
In the literature, common bandwidth choices include the median heuristic \cite{liu2016stein}, as well as rule-of-thumb bandwidth formulas motivated by asymptotic optimality of kernel estimators \cite{wibisono2024optimal}. 
For the Gaussian kernel used in this paper, the theoretical analysis suggests a scaling of the form
\begin{equation}\label{eq:h_scaling_guideline}
h_N = c\, \ell_N \, N^{-1/(d+4)},
\end{equation}
where $\ell_N$ is the median of pairwise distances and $c$ is a constant to be determined by numerical trials.

In our formulation, however, $h$ 
enters the discrete free energy $\mathcal{F}_h$ explicitly.  Consequently, adjusting $h$ dynamically during the iteration process would result in a time-varying energy functional, which complicates the interpretation of the gradient flow and the convergence to a fixed stationary distribution. To ensure mathematical consistency and to maintain the stability properties established for ImEQ scheme, we keep $h$ constant throughout all main experiments. The specific value of $h$ is selected based on the scaling guideline in \eqref{eq:h_scaling_guideline} and several trials.

\subsection{Toy examples}
We begin by testing ImEQ method on three toy examples, which are commonly used as benchmark tests in previous studies \cite{Chen2018, liu2019understanding, rezende2015variational}. The three target distributions are as follows:
\begin{itemize}
\item {\bf Double-banana-shaped target distribution:} 
$$
\rho(\x) \propto \exp \bigg\{ -\frac{1}{2} |\x|^2 - \frac{1}{2} \bigg( \ln [ x_1^2 + 100 (x_2 - x_1^2)^2] - \ln 30 \bigg)^2 \bigg \}
$$
with $\x = (x_1, x_2)$, adapted from Refs. \cite{liu2019understanding, rezende2015variational}. 

\item  {\bf Star-shape target distribution:} 
A five-component Gaussian mixture \cite{wang2019stein}:
\begin{equation}\label{stardistribution}
\rho(\x) \propto \frac{1}{5}\sum_{i=1}^5 N(\bm x |\bm \mu_i,\bm \Sigma_i), 
\end{equation}
where
\[
\bm \mu_i= 
\begin{pmatrix}
\cos\left(\frac{2\pi}{5}\right) & -\sin\left(\frac{2\pi}{5}\right)\\
\sin\left(\frac{2\pi}{5}\right) & \cos\left(\frac{2\pi}{5}\right)
\end{pmatrix}^{i-1}
\begin{pmatrix} 1.5 \\ 0 \end{pmatrix},~~\bm \Sigma_i= \begin{pmatrix}
\cos\left(\frac{2\pi}{5}\right) & -\sin\left(\frac{2\pi}{5}\right)\\
\sin\left(\frac{2\pi}{5}\right) & \cos\left(\frac{2\pi}{5}\right) 
\end{pmatrix}
^{i-1}
\begin{pmatrix}
1 & 0 \\
0 & 0.01 \\
\end{pmatrix}.
\]

\item  
{\bf Eight-component Gaussian mixture \cite{chen-wangEVI}:}   
\[
\rho(\x) \propto \frac{1}{8}\sum_{i=1}^8 N(\bm x|\bm \mu_i,\bm \Sigma),
\]
where $\bm \mu_1=(0,4)$, $\bm \mu_2=(2.8,2.8)$, $\bm \mu_3=(4,0)$, $\bm \mu_4=(-2.8,2.8)$, $\bm \mu_5=(-4,0)$, $\bm \mu_6=(-2.8,-2.8)$, $\bm \mu_7=(0,-4)$, $\bm \mu_8=(2.8,-2.8)$, and $\bm \Sigma=\left[\begin{array}{rr} 0.2 & 0 \\ 0 & 0.2 \end{array}\right]$. 

\end{itemize}
All target distributions are known up to the normalizing constant. 
Throughout this section, we refer to the step size $\tau$ used in EVI-Im and ImEQ as the learning rate, denoted by ${\rm lr}$, to be consistent with the terminology used in other methods.

\begin{figure*}[!ht]
  \centering
\includegraphics[width= \linewidth] {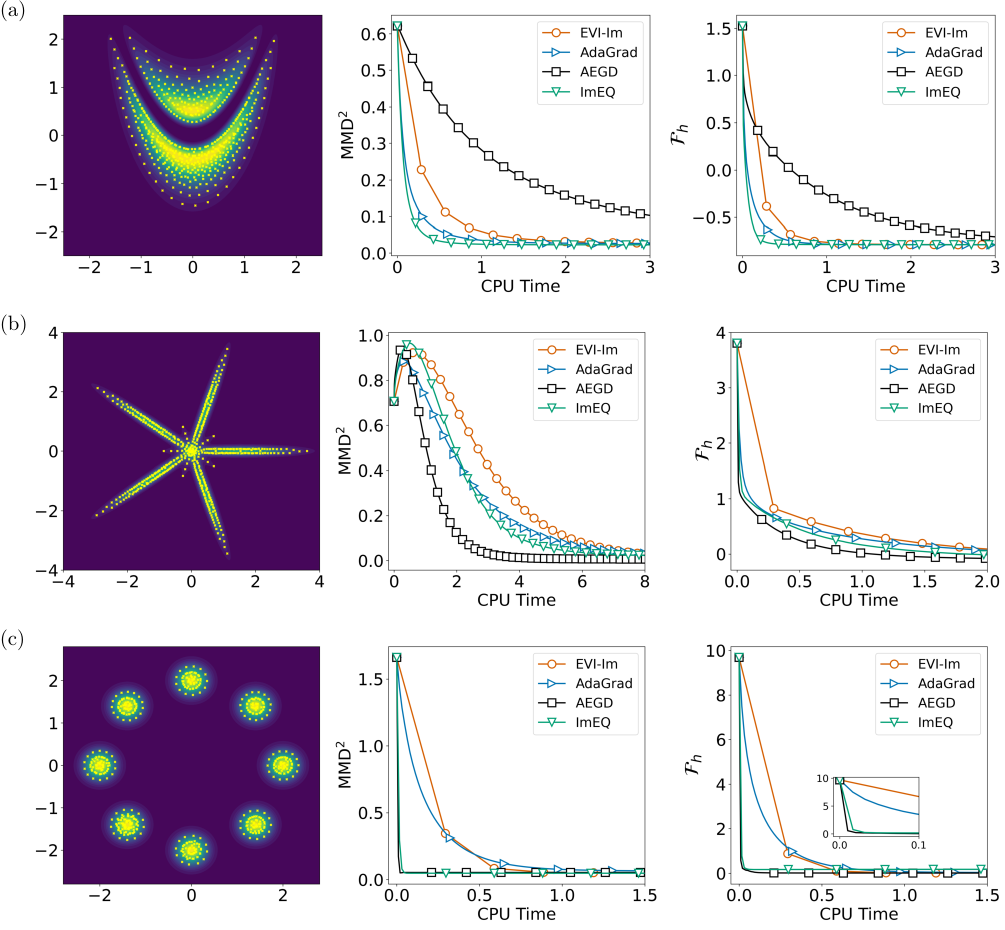}
 \caption{ "Double-banana" (a), "Star" (b) and "Eight-component" (c) cases: particles obtained by ImEQ method after 200 iterations (left); plot of MMD$^2$ (middle) and $\mathcal{F}_h$ (right) with respect to CPU time for different methods. 
 For AdaGrad and EVI-Im methods, ${\rm lr}=0.1$ in all cases. 
 In the case of ImEQ method, ${\rm lr}=0.01$ for "Double-banana" and "Star" cases, while ${\rm lr}=0.1$ for "Eight-component" case. For AEGD method, ${\rm lr}=0.001$ for "Double-banana" case, ${\rm lr}=0.01$ for "Star" case and ${\rm lr}=0.1$ for "Eight-component" case. Here, markers are displayed every 20 data points for AdaGrad, AEGD, and ImEQ curves.
 }
\label{doublebanana}
\end{figure*}

For the toy examples, we initialize the particles by sampling from a two-dimensional standard Gaussian distribution.  
The number of particles is $N = 500$, and the kernel bandwidth is $h = 0.1$. Since our primary focus is on comparing the optimization performance of each method, we do not attempt a systematic study of bandwidth tuning. Throughout the experiments, we choose $h$ using the scaling guideline \eqref{eq:h_scaling_guideline} with a small number of pilot runs, and then keep $h$ fixed.
For implicit methods (EVI-Im and ImEQ), we use the gradient descent with the Barzilai-Borwein method \cite{barzilai1988two} to solve the optimization problem at each time step. 
The maximum number of iterations for the inner optimization loop is set to $K=20$, as obtaining the optimal solution at the initial stage of these methods is unnecessary when we are primarily interested in the equilibrium rather than the dynamics. Additionally, for energetic quadratization methods (AEGD and ImEQ), we set the constant $C = 5$.  
We have experimented with different learning rates $\rm lr = 0.1, 0.01, 0.001, 10^{-4}$ and selected the optimal learning rate for each method based on performance in our tests.

To compare the performance of the above methods, we track two complementary quantities as functions of CPU time. 
First, we plot the evolution of the discrete free energy $\mathcal{F}_h(\{\x_i\}_{i=1}^N)$, since $\mathcal{F}_h$ is the discrete objective associated with the interacting-particle gradient flow and reflects the optimization behavior of the methods. 
In particular, for ImEQ, the theoretical analysis is formulated in terms of a modified energy, which serves as a surrogate of $\mathcal{F}_h$ in the small-step regime. 
Therefore, the evolution of $\mathcal{F}_h$ provides a practical way to assess how the numerical iteration behaves with respect to the original objective.

Second, we plot the evolution of the squared Maximum Mean Discrepancy (MMD$^2$), which serves as an additional metric for comparing two empirical distributions \cite{gretton2012kernel, arbel2019maximum}:
$$
\mbox{MMD}^2(\{ \x_i \}_{i=1}^N, \{ \y_j \}_{j=1}^M) = \frac{1}{N^2} \sum_{i, j =1}^N k(\x_i, \x_j) + \frac{1}{M^2} \sum_{i, j =1}^M k(\y_i, \y_j) - \frac{2}{NM}  \sum_{i=1}^N  \sum_{j=1}^M  k(\x_i, \y_j),
$$
with respect to CPU time. 
Here, we use a polynomial kernel $k(\x, \y) = (\x^T \y /3 +1)^3$. The set $\{ \x_i\}_{i=1}^N$ consists of the $N$ particles generated by different ParVI methods, while $\{ \y_j\}_{j=1}^M$ contains $M=5000$ reference samples drawn from $\rho^*$ using Langevin Monte Carlo (LMC). 
Thus, MMD$^2$ provides a quantitative measure of sampling quality over time, complementary to the objective-based quantity $\mathcal{F}_h$.

The first column in Fig. \ref{doublebanana}(a)-(c) shows the particles obtained from ImEQ after 200 iterations, while columns 2 and 3 show the MMD$^2$ and discrete free energy $\mathcal{F}_h$ of each method as functions of CPU time for all three cases.  
Since the final particle distributions are similar across different methods, we present only the results from ImEQ in the first column.  
As shown in Fig. \ref{doublebanana}, the proposed ImEQ method exhibits superior computational efficiency compared to EVI-Im and AdaGrad across all three cases. In the ``Star'' and ``Eight-component'' cases, AEGD achieves the best CPU time performance: since it can use the same step size as ImEQ in these cases (${\rm lr}=0.01$ and ${\rm lr}=0.1$, respectively) while requiring no inner optimization loop, its fully explicit nature makes it faster per iteration. In contrast, in the ``Double-banana'' case, AEGD requires a much smaller step size (${\rm lr}=0.001$) to converge stably, whereas ImEQ remains stable at ${\rm lr}=0.01$; this is why ImEQ is more efficient than AEGD in that case. Overall, ImEQ offers efficiency comparable to AEGD while being more robust to the choice of step size, as further illustrated in Fig.~\ref{doublebanana2}.

Compared with EVI-Im, ImEQ significantly reduces the CPU time when the same learning rate is used.
To further quantify the computational efficiency and accuracy of ImEQ relative to EVI-Im, Table~\ref{comparediffN} reports the steady-state MMD$^2$, the steady-state discrete free energy $\mathcal{F}_h$, and the corresponding CPU time for different numbers of particles. Here, the steady state is defined as the first iterate for which the change in $\mathcal{F}_h$ between two consecutive iterations is below $10^{-5}$.
In this comparison, we set the learning rate to ${\rm lr}=0.01$ for both methods. We note that EVI-Im can often be run with a larger step size (e.g., ${\rm lr}=0.1$ in Fig.~\ref{doublebanana}), but we use the same ${\rm lr}$ here to enable a direct per-iteration cost comparison.
Overall, ImEQ achieves comparable MMD$^2$ values while requiring substantially less CPU time for particle numbers $N\ge 100$.
Moreover, the speedup becomes more pronounced as $N$ increases.

  \begin{table}[!h]
    \centering
    \caption{Comparison of ImEQ and EVI-Im methods for CPU time (seconds), MMD$^2$ and $\mathcal{F}_h$  with different particle numbers in the ``Double-banana'' case.} 
    \begin{tabular}{|c|c|c|c|c|c|c|c|c|c|}
    \hline
     &  \multicolumn{3}{|c|}{N=100} & \multicolumn{3}{|c|}{N=200} & \multicolumn{3}{|c|}{N=500} \\
    \hline
    \  & Time & MMD$^2$ & $\mathcal{F}_h$ & Time & MMD$^2$ & $\mathcal{F}_h$  &  Time & MMD$^2$ & $\mathcal{F}_h$ \\
    \hline
    ImEQ & {\bf 0.16} & 0.020 & -0.625 & {\bf 0.34} & 0.024 & -0.727 & {\bf 1.46} & 0.023 & -0.789  \\
    \hline
    EVI-Im & 2.31 & 0.022 & -0.628  & 6.76 & 0.025 & -0.727 & 36.61 & 0.027 & -0.790  \\
    \hline
    \end{tabular}
    
    \label{comparediffN}
    \end{table}

As mentioned in previous sections, both AEGD and ImEQ need to take a smaller step size for the algorithm to work in some cases. Fig.~\ref{doublebanana2} compares the performance of the two algorithms on the ``Double-banana'' example across a range of step sizes. We observe that when the step size is below $10^{-2}$, ImEQ is able to converge to a local minimizer of the discrete free energy $\mathcal{F}_h$. In contrast, AEGD typically requires a smaller step size, below $10^{-3}$, to reach a comparable minimizer; for larger step sizes it tends to become unstable. This suggests that ImEQ is more robust, by treating the potential term implicitly.
\begin{figure*}[!ht]
\centering
 \begin{minipage}{0.45 \linewidth}
  \begin{overpic}[width=\linewidth]{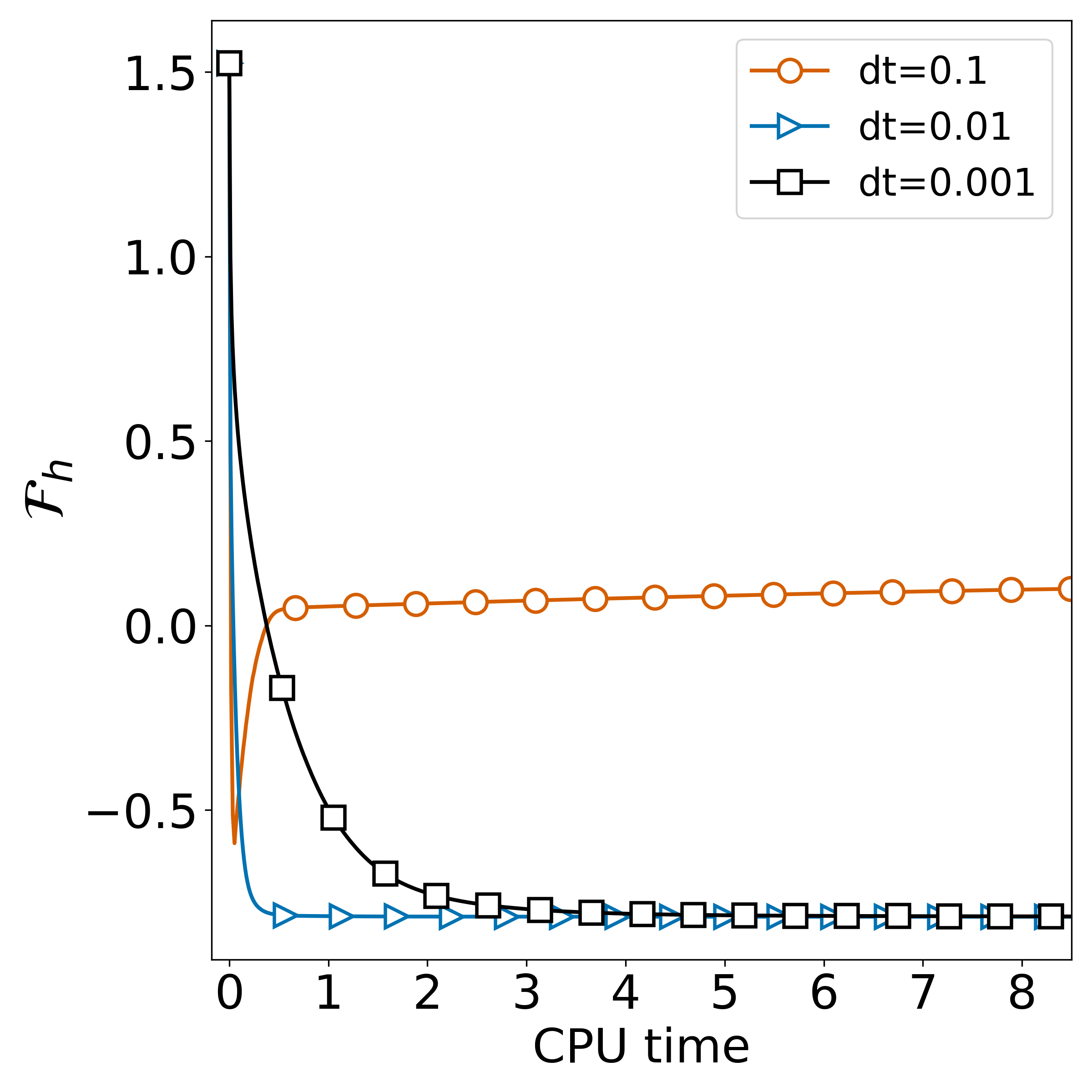}
    \put(-3, 94){{\scriptsize (a)}}
    \end{overpic}
   \end{minipage}
   \hspace{1.5 em}
    \begin{minipage}{0.45 \linewidth}
       \begin{overpic}[width=\linewidth]{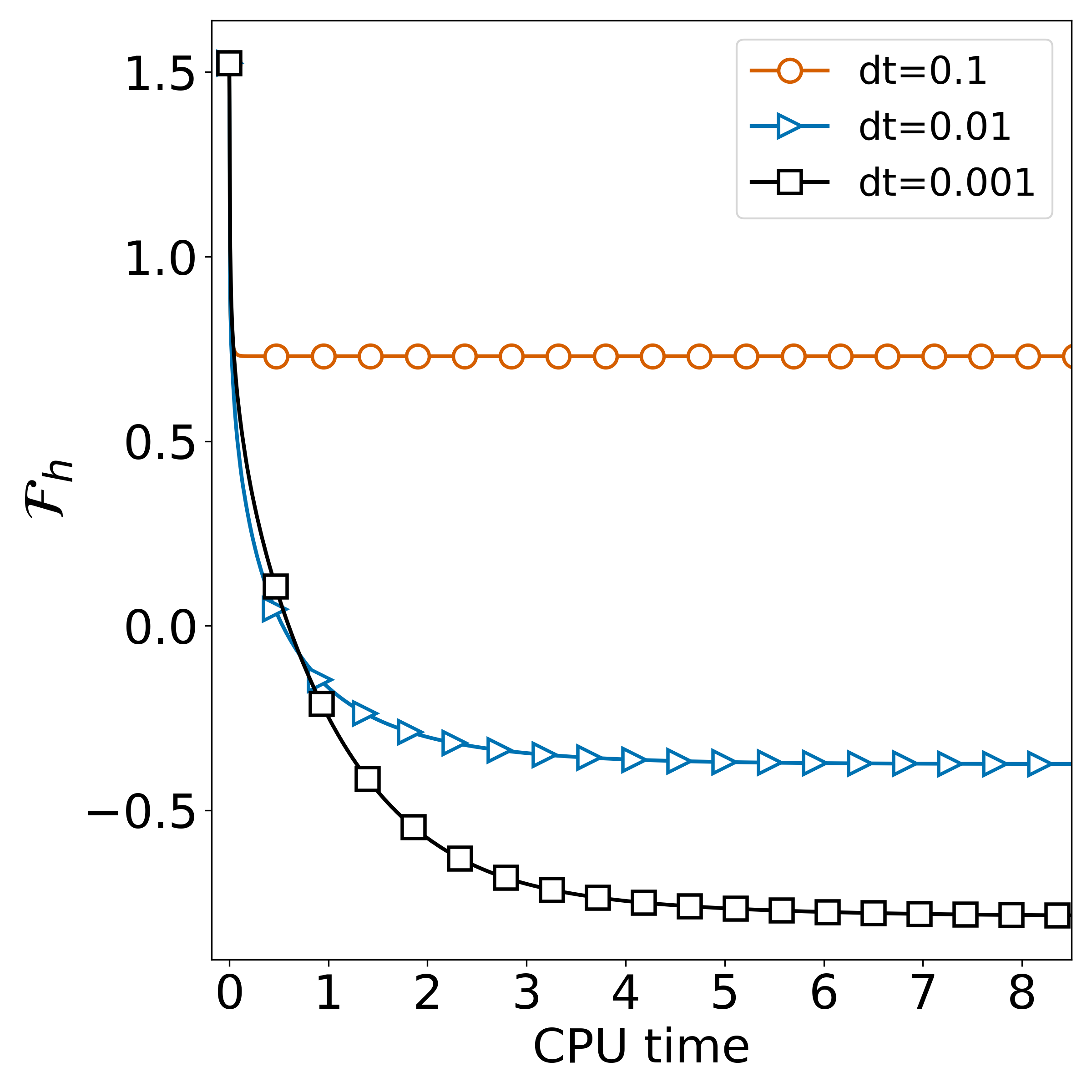}
    \put(-3, 94){{\scriptsize (b)}}
    \end{overpic}
 \end{minipage}
 \caption{(a): $\mathcal{F}_h$ with respect to CPU time for different learning rates for ImEQ. (b): $\mathcal{F}_h$ with respect to CPU time for different learning rates for AEGD. Here, markers are displayed every 20 data points for different curves.}\label{doublebanana2}
\end{figure*}

Next, we consider a more challenging task by initializing the particle distributions for all methods significantly away from the target distribution. This setup better reflects real-world applications, where the mean or variance of the target distribution is often unknown or difficult to estimate. Specifically, we define the target distribution as a star-shaped, five-component Gaussian mixture \eqref{stardistribution}, and the initial distribution as $\mathcal{N}((5, 5), \mathsf{I})$. 

\begin{figure*}[ht]
  \centering
 \begin{minipage}{0.65 \linewidth}
\includegraphics[width = \linewidth]{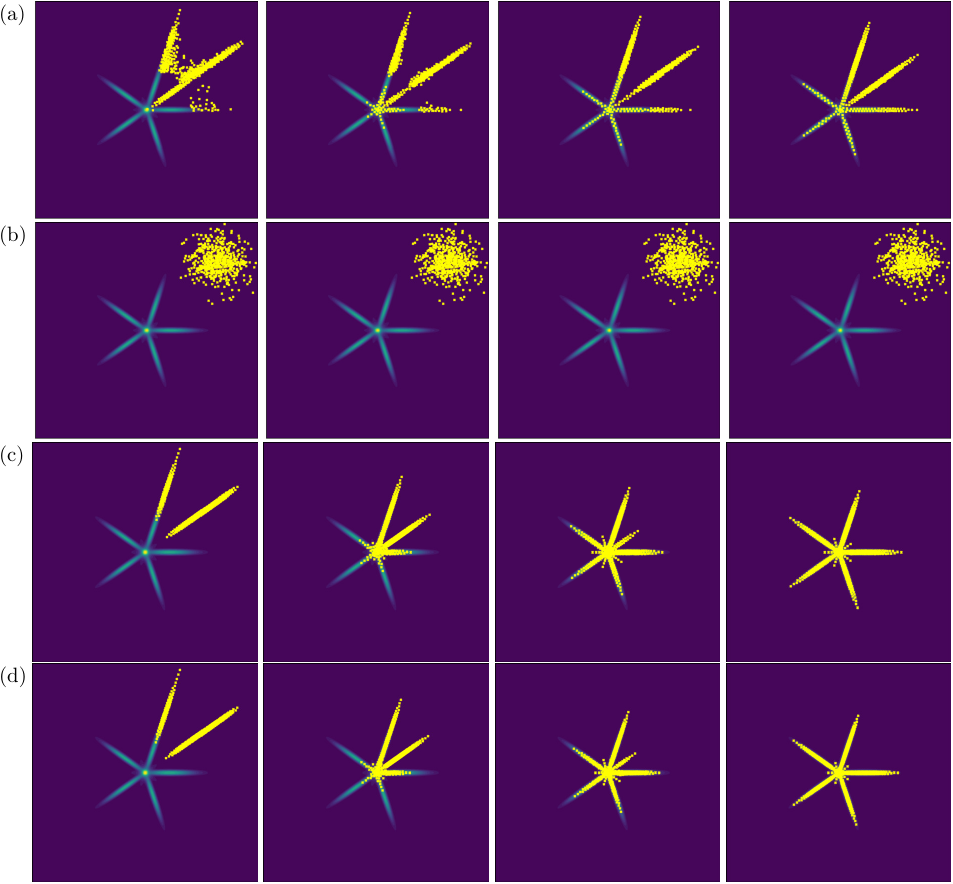} 
\end{minipage}
 \hfill
 \begin{minipage}{0.3 \linewidth}
  \begin{overpic}[width=\linewidth]{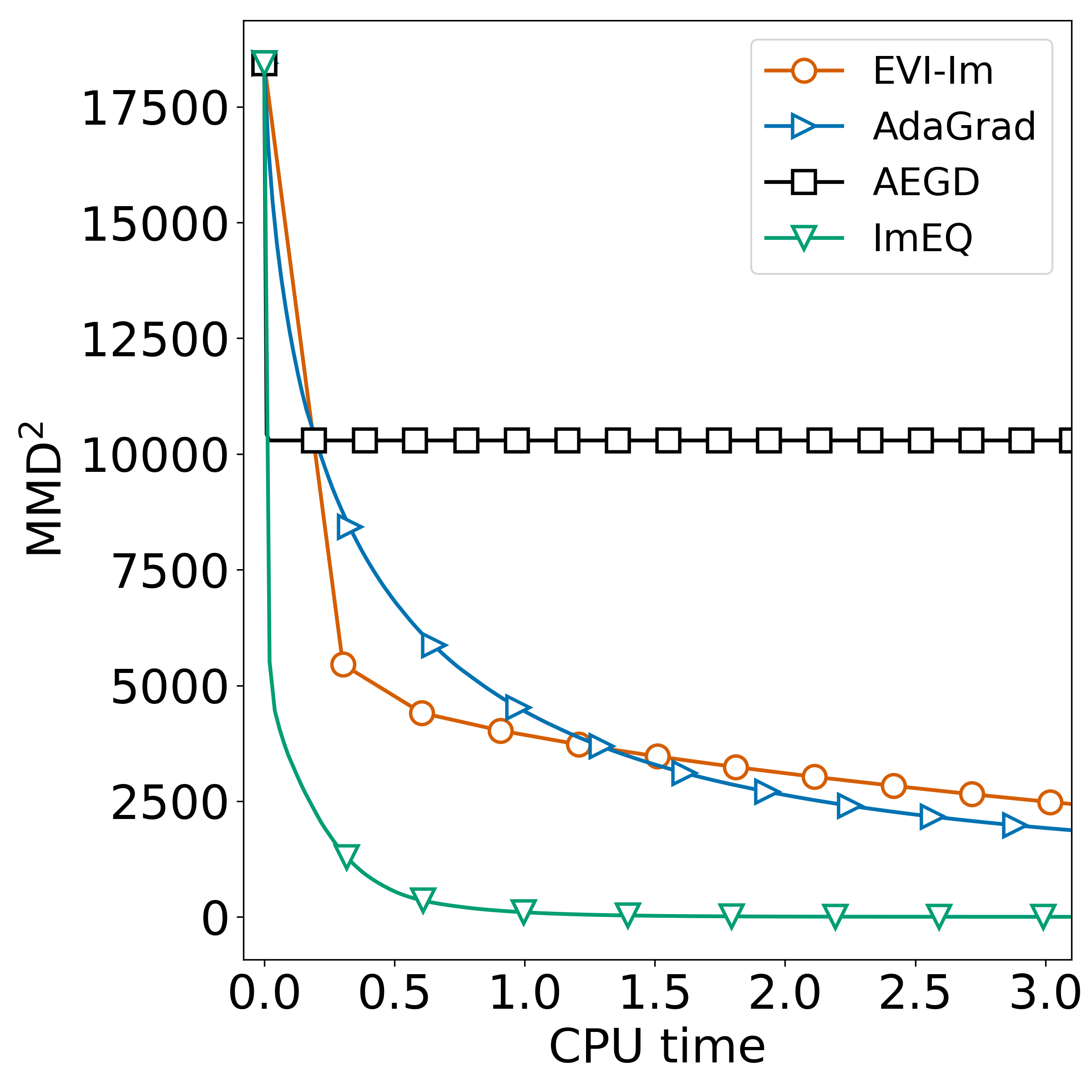}
    \put(-2, 92){{\scriptsize (e)}}
    \end{overpic}
   
     \vspace{0.5 em}
       \begin{overpic}[width=\linewidth]{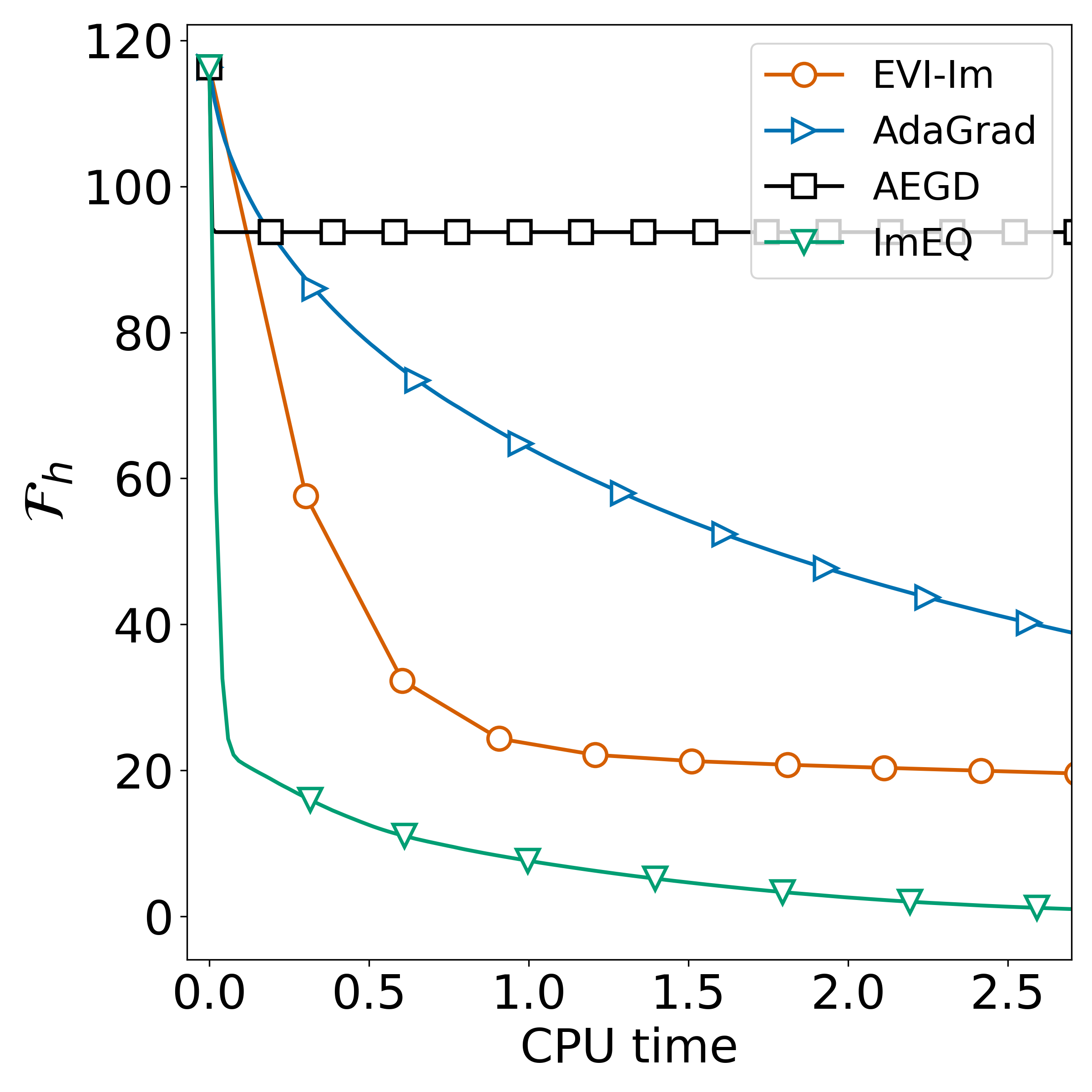}
    \put(-2, 92){{\scriptsize (f)}}
    \end{overpic}
 \end{minipage}
\caption{``Star'' case with the initial distribution set as a Gaussian distribution with a nonzero mean.  
(a)-(b): Particles obtained by AdaGrad and AEGD at iterations 500, 1000, 2000, and 5000 (from left to right).  
(c)-(d): Particles obtained by EVI-Im and ImEQ at iterations 20, 100, 200, and 500 (from left to right).  
(e)-(f): Plots of MMD$^2$ and $\mathcal{F}_h$ as functions of CPU time for different methods. Markers are displayed every 20 data points for AdaGrad, AEGD, and ImEQ curves.}
  \label{star_500_1}
\end{figure*}

Fig. \ref{star_500_1}(a)-(d) show the distribution of particles generated by different methods at various iterations with $N = 500$. The learning rate is set to 0.1 for AdaGrad and 0.01 for all other methods.
The results show that EVI-Im and ImEQ achieve similar approximations to the target distribution at different iterations, and their dynamics are nearly identical. However, ImEQ is much more efficient in terms of CPU time. Both methods perform better than AdaGrad and AEGD. Interestingly, in this case, AEGD fails to explore the star-shaped target distribution. We also tested various other learning rates (${\rm lr} = 0.1, 0.01, 0.001, 10^{-4}$) for AEGD, which yielded nearly identical results. This shows that keeping the potential part $H = \frac{1}{N} \sum_{i=1}^N V(\x_i)$ implicit, as in ImEQ, is necessary when the initial distribution is far from the target distribution.

To evaluate ImEQ on a heavy-tailed target, we consider a 2D Student’s t distribution:
$$\rho^*(\x) \propto \left(1 + \frac{\|\x - \bm{\mu}\|^2}{\nu s^2}\right)^{-\frac{\nu+2}{2}} \mbox{with } \bm{\mu}=\bm{0}, s = 1, \nu = 3,$$
which leads to 
$$
V(\x) = \frac{5}{2} \ln \left(1 + \frac{x_1^2 + x_2^2}{3}\right), \mbox{with } \x=(x_1, x_2). 
$$
We initialize particles from a standard Gaussian. We set the number of particles to $N = 500$, the kernel
bandwidth to $h = 0.4$ (chosen based on \eqref{eq:h_scaling_guideline} and a few trials), the constant $C=10$ and the learning rate $\rm lr = 0.01$ as before.

Since tail coverage is a primary diagnostic for heavy-tailed targets, we compute the empirical tail probability
\[
\widehat P_R=\frac{1}{N}\sum_{i=1}^N \mathbf{1}_{\{\|x_i\|>R\}},
\]
and compare it with the ground-truth tail probability $P_R^{\mathrm{True}}:=\mathbb{P}_{\rho^*}(\|X\|>R)$.
The ground-truth values are estimated by Monte Carlo sampling using $10^6$ i.i.d.\ samples from $\rho^*$:
$$
P_2^{True} = 0.2808, \ P_3^{True} = 0.1252, \ P_4^{True} = 0.0627, \ P_5^{True} = 0.0352.  
$$
At the steady state, the estimated tail coverages are
$$
\hat{P}_2 = 0.268, \ \hat{P}_3 = 0.096, \ \hat{P}_4 = 0.048, \ \hat{P}_5 = 0.000.  
$$

The particles produced by ImEQ at the steady state are shown in Fig.~\ref{studentdistribution}(a). Figs.~\ref{studentdistribution}(b--c) report the evolution of $\mathrm{MMD}^2$ and the discrete KL-type objective $\mathcal{F}_h$ versus CPU time. When computing $\mathrm{MMD}^2$ for this heavy-tailed case, we employ the inverse multiquadric (IMQ) kernel
\[
k(x,y)=(c+\|x-y\|^2)^{-\beta},\qquad c=1,\ \beta=\tfrac12,
\]
which is commonly used as a tail-sensitive discrepancy measure.

Overall, this example indicates that ImEQ remains stable and provides reasonable tail coverage for moderate thresholds (e.g., $R\le 4$) for $N = 500$. We also emphasize that combining a KL divergence with a Gaussian kernel is not ideal for deep-tail characterization: KL-based criteria can under-emphasize rare tail events in finite-sample regimes, while Gaussian kernels are inherently local and may not sufficiently promote exploration of far-tail regions. Importantly, the ImEQ framework is not tied to this particular choice. The same methodology can be formulated with alternative divergences/energies (e.g., tail-sensitive discrepancies) and with kernels better suited to heavy-tailed geometries (e.g., heavier-tailed kernels such as IMQ), potentially improving deep-tail coverage. A systematic investigation of such extensions is beyond the scope of this article and is left for future work.

\begin{figure}[ht]
\centering
\begin{overpic}[width=\linewidth]{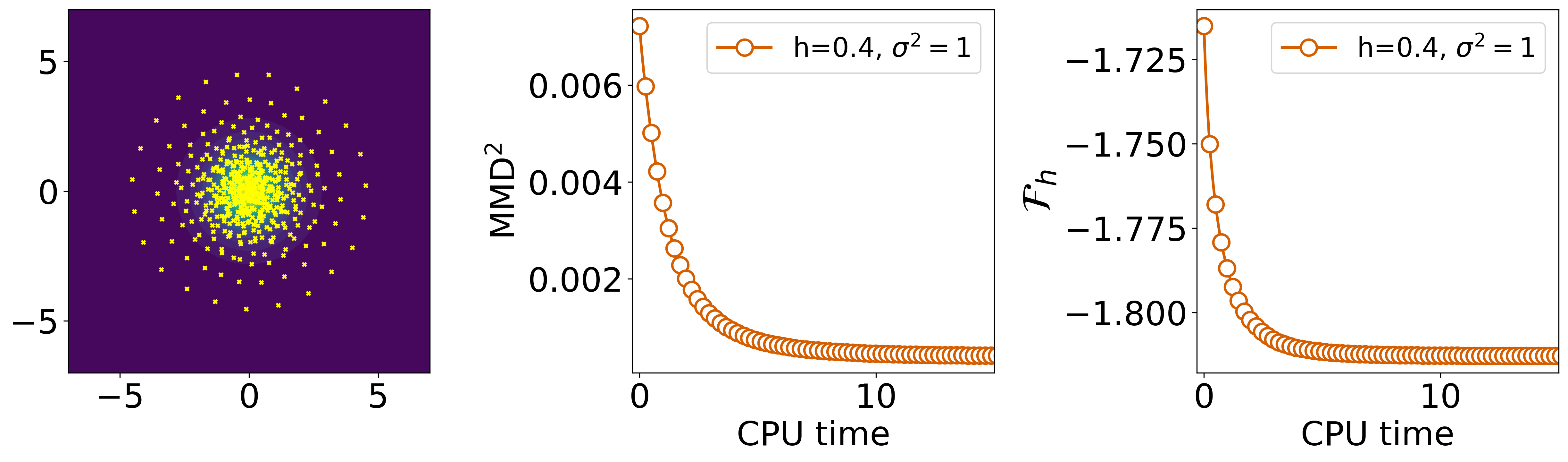}
 \put(0, 30){{\scriptsize (a)}}
  \put(30, 30){{\scriptsize (b)}}
   \put(70, 30){{\scriptsize (c)}}
    \end{overpic}
 \caption{(a):  Particles obtained by ImEQ at the steady state; (b): Plot of IMQ ${\rm MMD}^2$ with respect to CPU time; (c): Plot of $\mathcal{F}_h$ with respect to CPU time.}
\label{studentdistribution}
\end{figure}

\subsection{Bayesian Logistic Regression with Real Data}\label{sec:BLR}
Consider a Bayesian logistic regression model for binary classification. Given the data set 
$\{ \mathbf{c}_t, y_t\}_{t=1}^{\tilde{N}}$ with $\tilde{N}$ the number of training entries, the logistic regression model is defined as  
\[
p(y_t = 1 \mid \mathbf{c}_t, \bm{\omega}) = \left[ 1 + \exp(-\bm{\omega}^T \mathbf{c}_t) \right]^{-1}.
\] 
The unknown parameters $ \bm \omega$ are the regression coefficients, whose prior is $N(\bm \omega; \mathbf{0}, \alpha \mathbf{I})$ with $\alpha = 1$. In the Bayesian logistic regression model, we compare the performance of ImEQ method with EVI-Im, AEGD, the classical AdaGrad and SVGD methods.
For all methods, we set the number of particles $N = 100$, 
the learning rate to 0.1, and the bandwidth to $h=0.1$. 
For ImEQ and AEGD methods, we set $C=5$.  
It is important to note that these parameters may not be optimal for all methods. 
Additionally, as noted in the context of ImEQ and EVI-Im methods,
we need to solve a minimization problem to update the positions of the particles at each time step (outer loop).
Following the approach outlined in Ref. \cite{wang2021particle}, for the case of Bayesian logistic regression with real data, it is not necessary for the algorithms to achieve exact local optimality in each iteration.
Thus, to minimize the functional in the inner loop, we employ the stochastic gradient descent method AdaGrad  \cite{duchi2011adaptive} with a learning rate of 
$\rm lr = 0.1$ at each time step. Again, we limit the maximum number of iterations for the inner loop to 20 to reduce computational cost.

\begin{figure}[htb]
  \centering
\begin{overpic}[width=\linewidth]{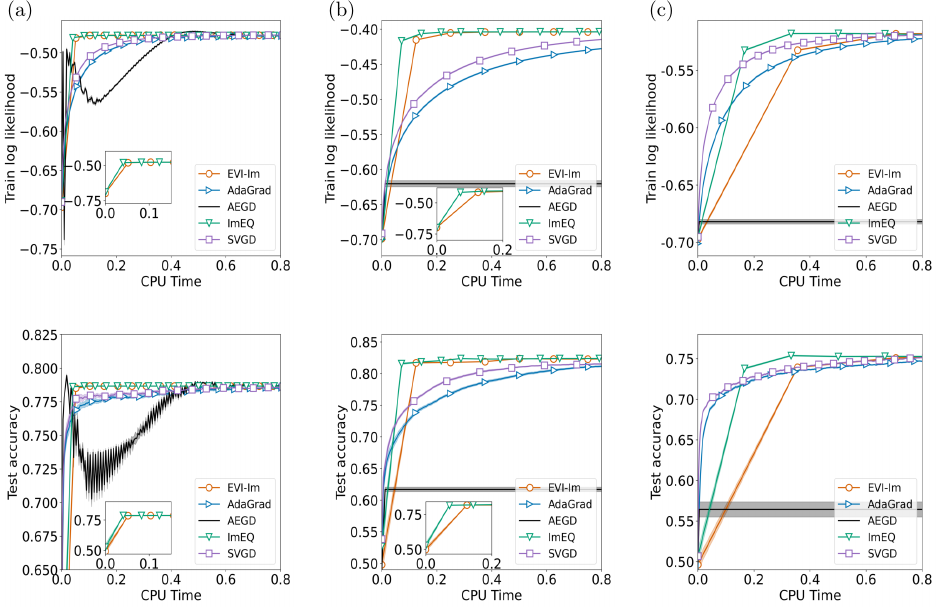}
    \end{overpic}
\caption{The training log-likelihood and test accuracy of the ``Diabetes'' (a), ``Image'' (b) and ``Covertype'' (c) datasets returned by different methods. Markers are displayed every 20 data points for AdaGrad and SVGD curves.}
\label{Bayeslr}
\end{figure}

Fig. \ref{Bayeslr} (a)-(b) show the training log-likelihood and test accuracy on datasets
``Diabetes'' (468 training entries, 8 features) and ``Image'' (1300 training entries, 18 features) 
for various methods with respect to the CPU time. For each method, 
a total of 20 simulations were performed. 
The $x$ axis represents the average CPU time across the 20 simulations for all methods under comparison. 
The solid lines indicate the mean values of test accuracy and training log-likelihood, while the shaded regions represent the standard error across the 20 simulations.

As shown in Fig. \ref{Bayeslr} (a)-(b), the test accuracies of different methods  converge towards similar values in the end. 
However, AEGD method exhibits significant fluctuations and larger standard errors, with notably lower test accuracy for the dataset ``Image''. Despite testing various learning rates for AEGD method ($\rm lr = 0.1, 0.01, 0.001, 10^{-4}$), we observed that they produced similar performance. 
Remarkably, ImEQ and EVI-Im methods demonstrate slightly higher log-likelihoods with lower CPU time. 
These methods also exhibit the smallest standard errors and fluctuations among the five methods, which may be attributed to the two-layer loop structure in both ImEQ and EVI-Im. 
Notably, ImEQ method shows a slight advantage over EVI-Im, achieving nearly the same log-likelihood with less CPU time. 
This is consistent with the results from the toy examples, where ImEQ method outperforms EVI-Im when the particle number $N\geq 100$.

We then consider a large dataset ``Covertype'' \cite{wang2019stein}, which contains 581,012 data entries and 54 features.  The prior of the unknown regression coefficients is also $p(\bm \omega) = N(\bm \omega; \mathbf{0}, \alpha \mathbf{I})$ with $\alpha = 1$.
Due to the large dataset size, the computation of log-likelihood $\nabla \ln \rho^*$ is expensive. Therefore, as noted in Ref. \cite{wang2021particle}, we randomly sample a batch of data to estimate a stochastic approximation of $\nabla \ln \rho^*$, and thus,  
the algorithms do not need to achieve the exact local optimality in each iteration. 
As before, for ImEQ and EVI-Im methods, we use the AdaGrad algorithm with learning rate $\rm lr = 0.1$ to minimize the functional in the inner loop, and we set the maximum number of iterations for the inner loop to 100. We also compare ImEQ method with EVI-Im, AEGD, the classical AdaGrad and SVGD methods, setting the bandwidth to $h=0.05$. All other parameters are as specified above. 
For each method, we conducted a total of 20 simulations, with the data randomly partitioned into training (80$\%$) and testing (20$\%$) sets in each simulation.

Fig. \ref{Bayeslr} (c) presents the test accuracy and training log-likelihood of the training data for each method with respect to the CPU time for the ``Covertype'' dataset. The results show that test accuracies of EVI-Im, SVGD, and AdaGrad methods converge to similar values, slightly lower than that of ImEQ method. While for AEGD method, it achieves relatively lower test accuracy and log-likelihood.
As before, we tested various learning rates ($\rm lr = 0.1, 0.01, 0.001, 10^{-4}$) for AEGD method, all yielding similar performance. 
Therefore, we only present the results with $\rm lr=0.1$ in this case. 
Based on these findings, we conclude that AEGD method does not show advantages in this Bayesian logistic regression setting. 
The proposed ImEQ method outperforms the other methods, particularly for the large data (``Covertype'' case), where it achieves relatively high test accuracy and log-likelihood with reduced CPU time.

\subsection{Bayesian Neural Network}

In this subsection, 
we compare the performance of ImEQ with other methods on Bayesian neural networks \cite{detommaso2018stein, liu2016stein}. The models are trained on the UCI datasets, following the experimental set-up described in Ref. \cite{liu2016stein}. 
Specifically, we employ neural networks with one hidden layer of 50 hidden units. 
All the datasets are randomly partitioned into $90\%$ for training and $10\%$ for testing, and the results are repeated across 30 random trials.

\begin{figure}[htb]
  \centering
  \begin{overpic}[width=0.45\linewidth]{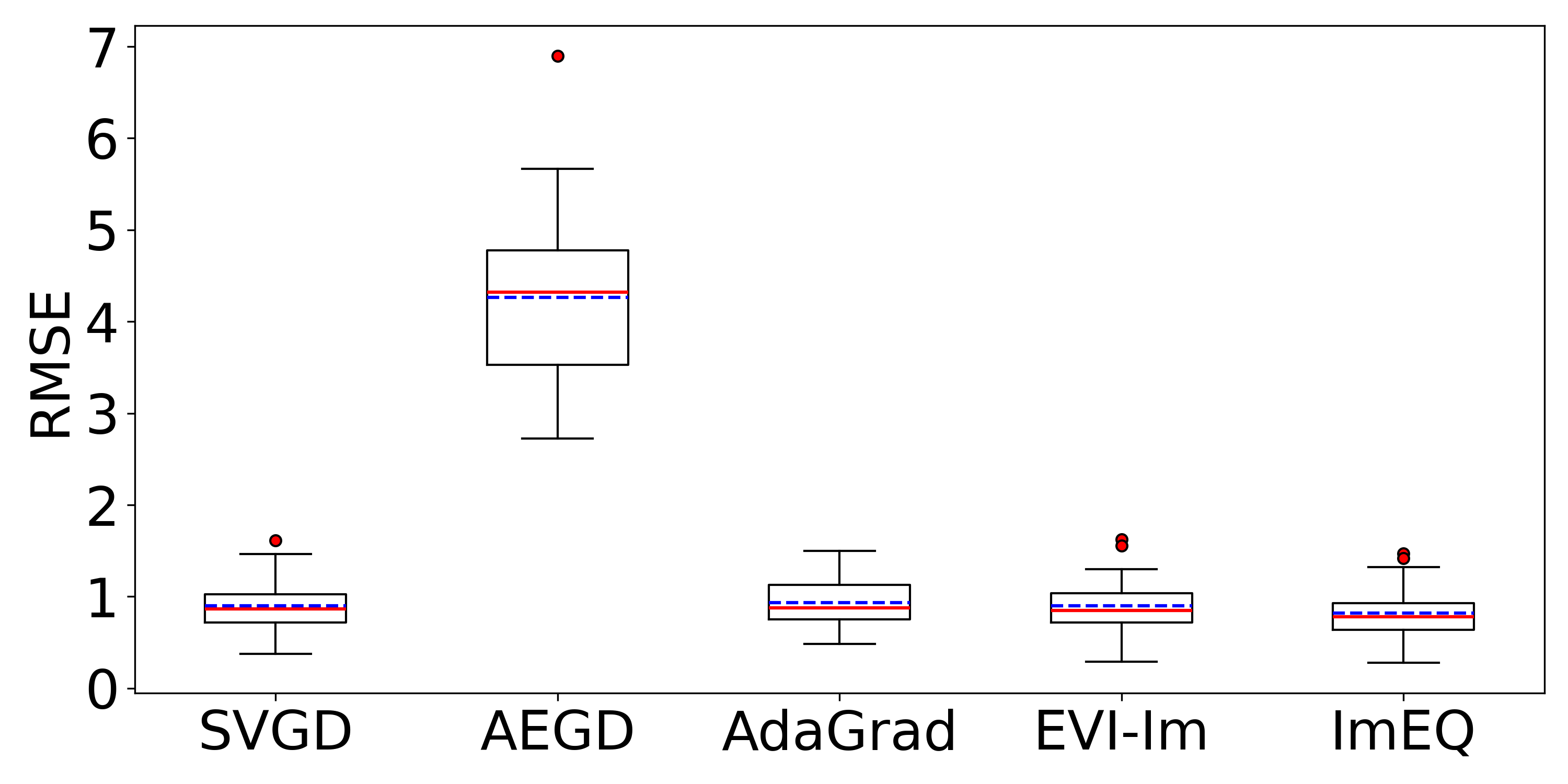}
  \put(-5, 50){(a)}
 \end{overpic}
   \includegraphics[width=0.45\linewidth]{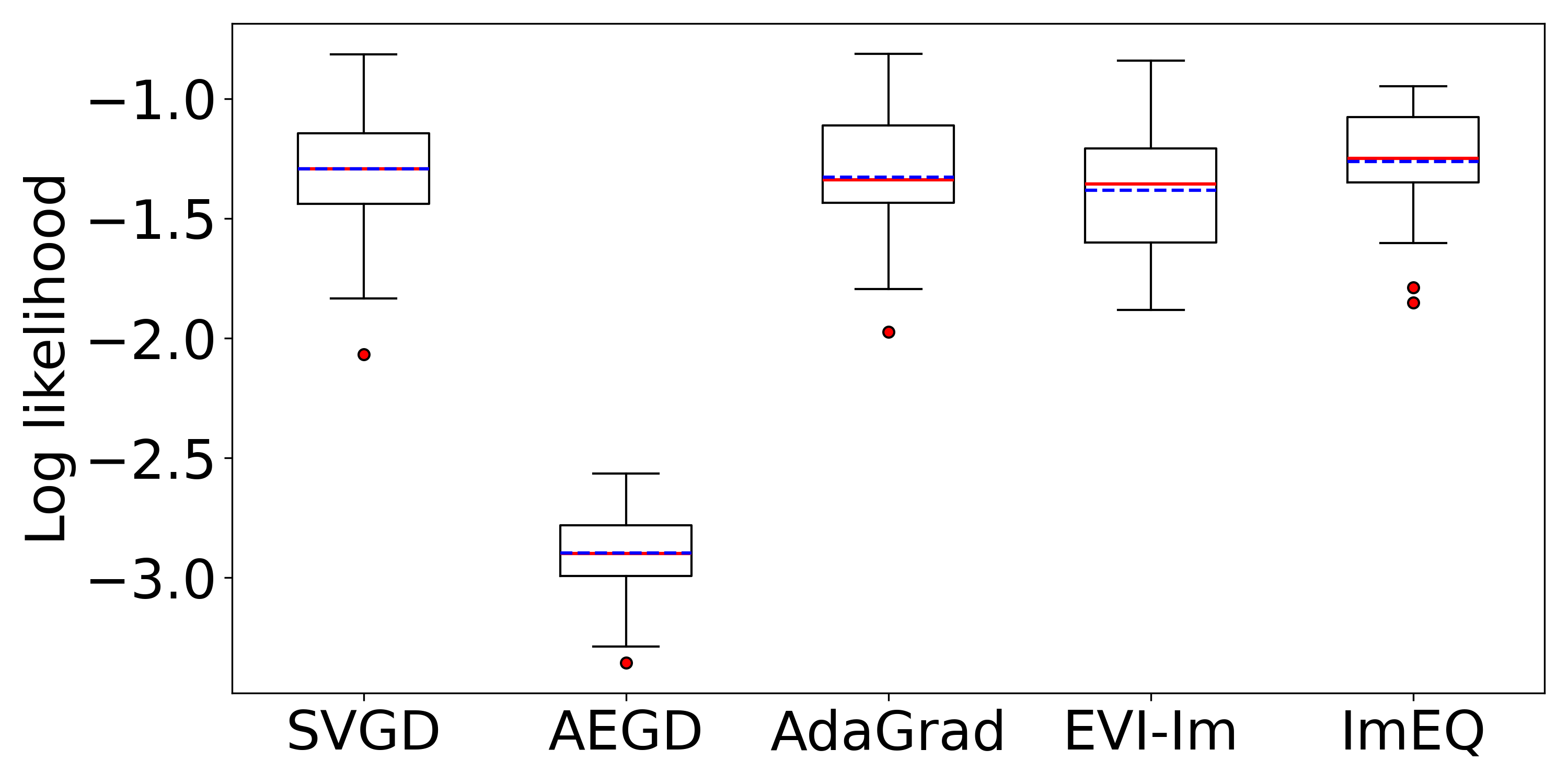}\\

\vspace{0.3cm}
   
  \begin{overpic}[width=0.45\linewidth]{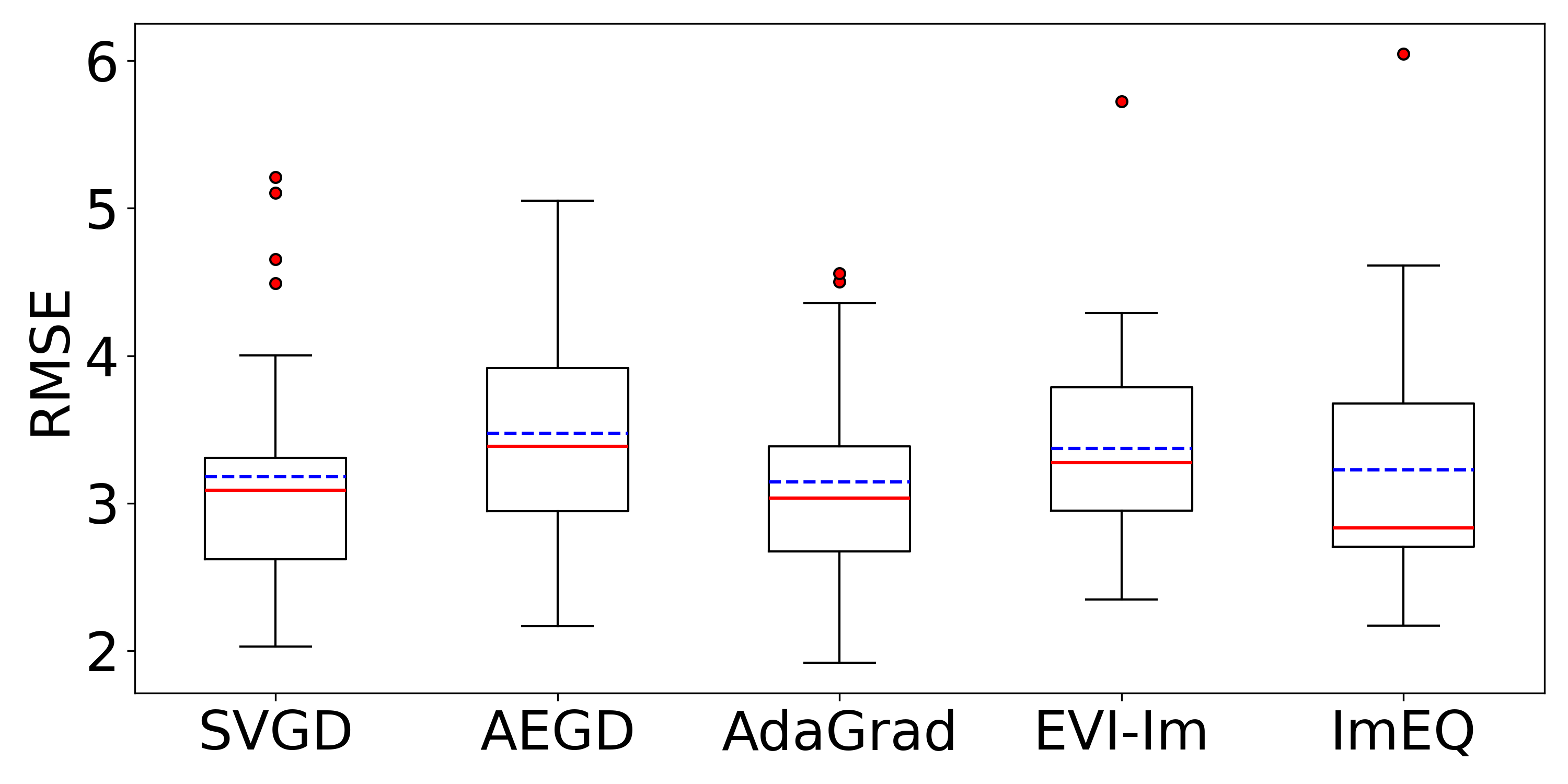}
  \put(-5, 50){(b)}
 \end{overpic}
   \includegraphics[width=0.45\linewidth]{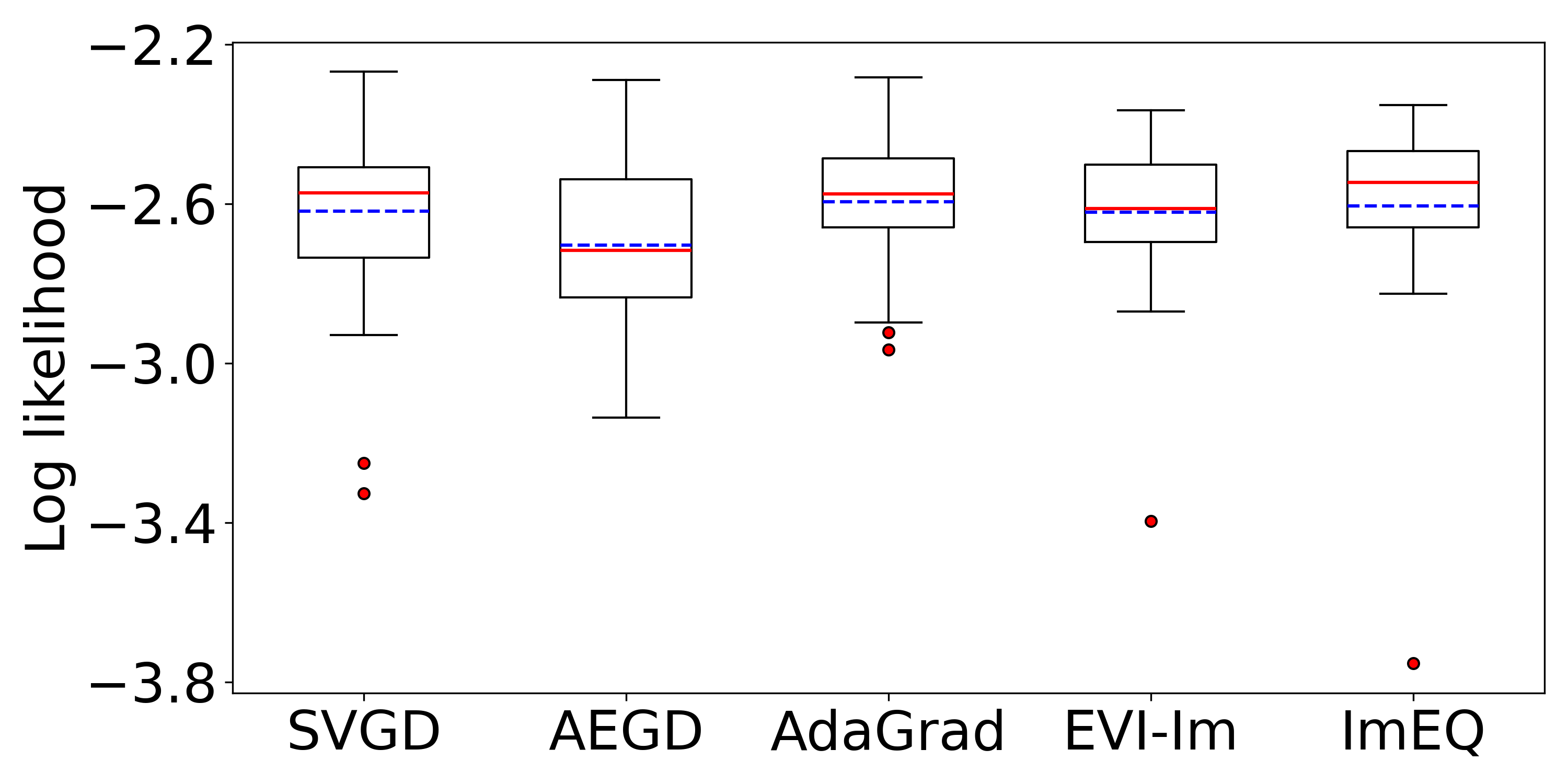}\\

\vspace{0.3cm}

   \begin{overpic}[width=0.45\linewidth]{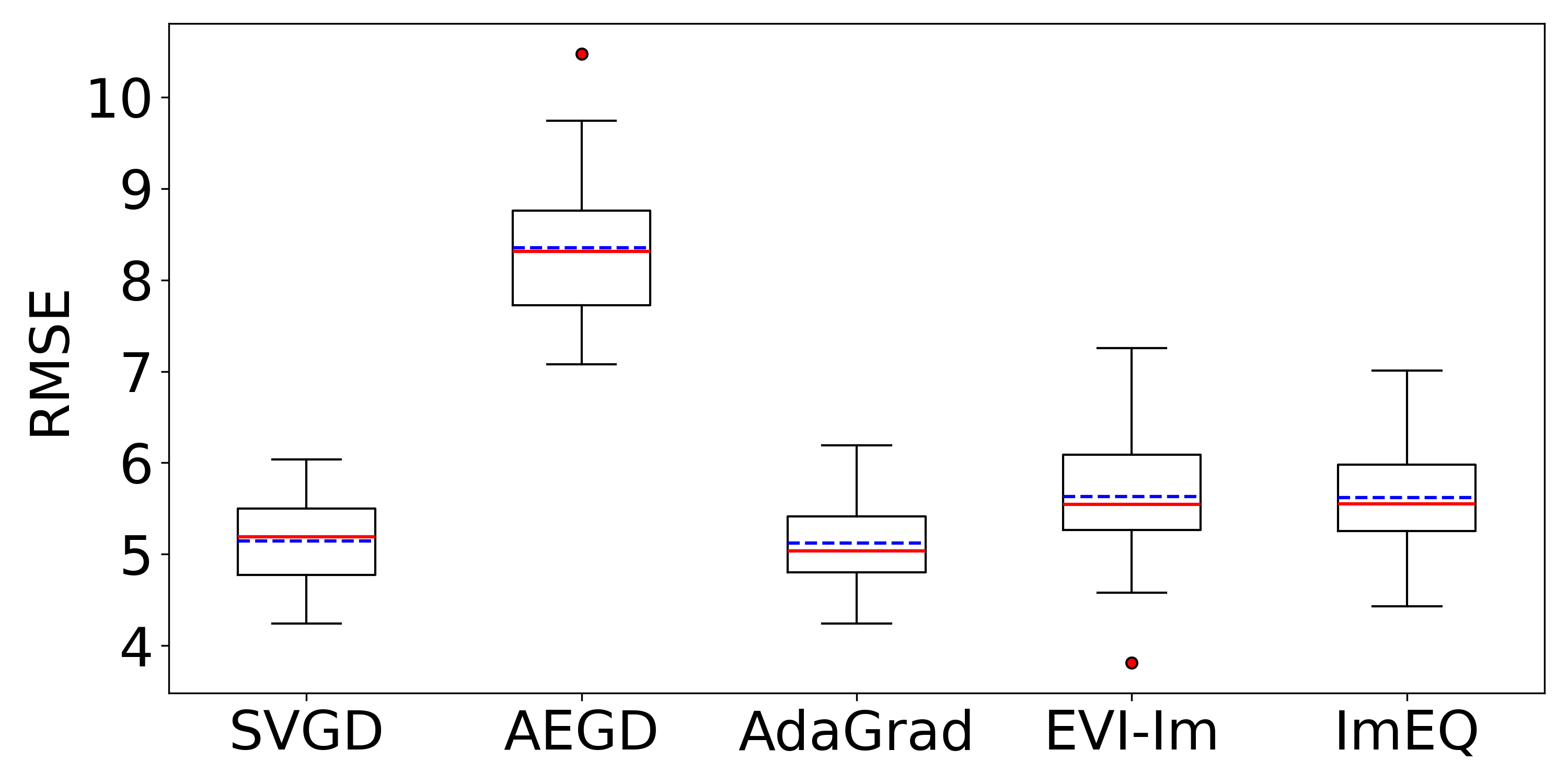}
   \put(-5, 50){(c)}
 \end{overpic}
   \includegraphics[width=0.45\linewidth]{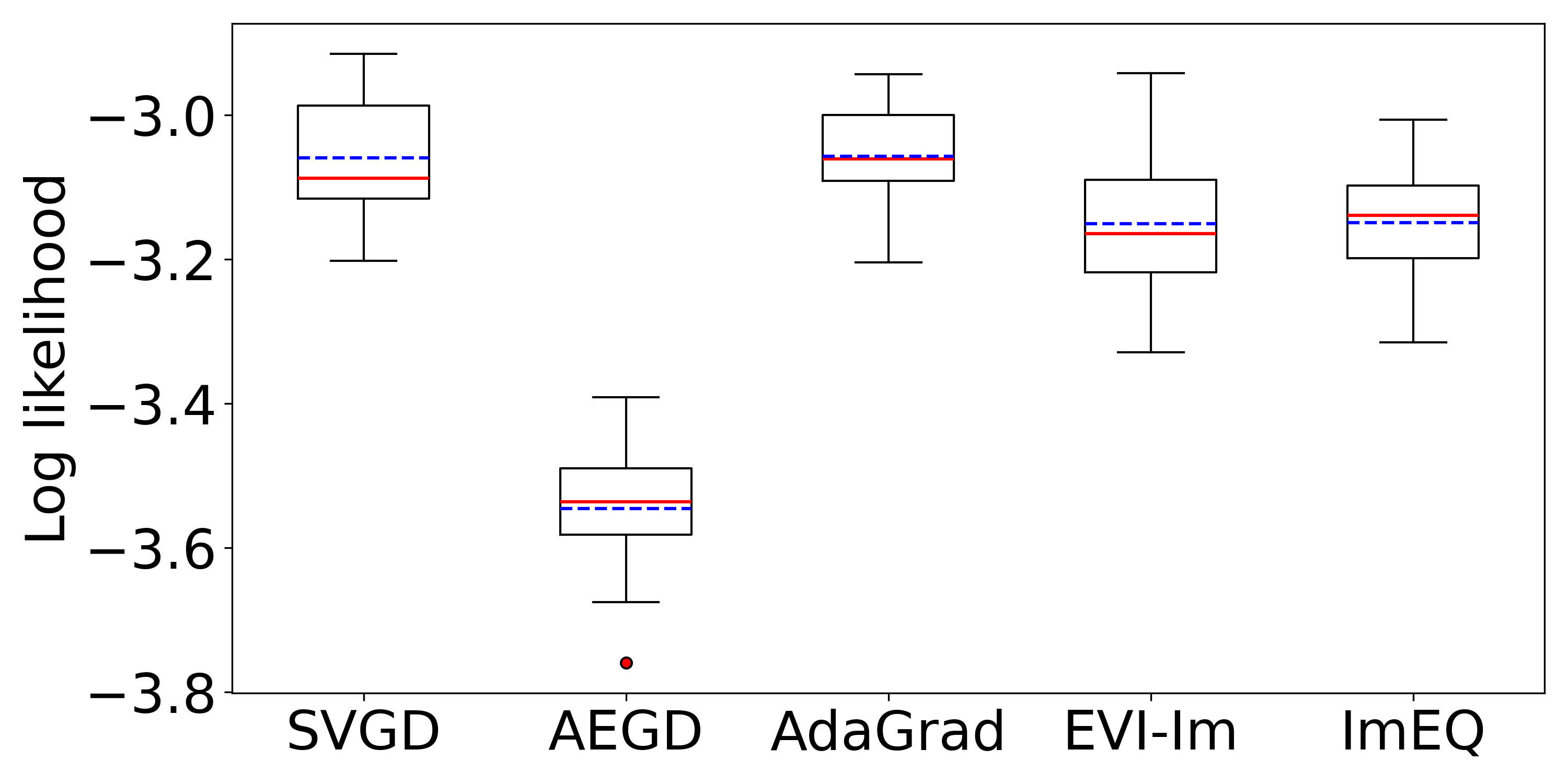}\\
   \caption{ Box plots of RMSE (left) and predictive log-likelihood (right) for different datasets: (a) ``Yacht Hydrodynamics'', (b) ``Boston Housing'', and (c) ``Concrete Data''.}
   \label{yacht}
\end{figure}

\begin{table}
  \centering
\caption{Comparison of RMSE, log-likelihood and average CPU time (seconds) for different methods.}  \label{bnnresult}
  \begin{tabular}{|c|c|c|c|c|c|}
  \hline
  \  & \multicolumn{5}{|c|}{\textbf{Avg. Test RMSE}}\\
  \hline
  \textbf{Dataset} & \textbf{SVGD} & \textbf{AEGD} & \textbf{AdaGrad} & \textbf{EVI-Im} & \textbf{ImEQ}\\
  \hline
  Yacht & 0.899 $\pm$ 0.051 & 4.265 $\pm$ 0.162 & 0.934 $\pm$ 0.048 & 0.900 $\pm$ 0.055  & 0.822 $\pm$ 0.050 \\
  \hline
  Boston & 3.178 $\pm$ 0.149 & 3.473 $\pm$0.133 & 3.143 $\pm$ 0.117  & 3.369 $\pm$ 0.123 &  3.226 $\pm$ 0.146 \\
  \hline
  Concrete & 5.141 $\pm$ 0.090 & 8.353 $\pm$ 0.144 & 5.119 $\pm$ 0.085  & 5.631 $\pm$ 0.133 &  5.621 $\pm$  0.109 \\
  \hline
  \  &  \multicolumn{5}{|c|}{\textbf{Avg. Test LL}} \\
  \hline
  Yacht & -1.293 $\pm$ 0.048 & -2.896 $\pm$ 0.034  & -1.327 $\pm$ 0.046  & -1.381 $\pm$ 0.047 & -1.262 $\pm$ 0.039 \\
  \hline
  Boston & -2.618 $\pm$ 0.043 & -2.704  $\pm$ 0.038 & -2.595 $\pm$ 0.029  & -2.620 $\pm$ 0.034 &  -2.605 $\pm$ 0.045 \\
  \hline
  Concrete & -3.059 $\pm$ 0.013 &-3.545  $\pm$ 0.014 &  -3.057 $\pm$ 0.012  & -3.150 $\pm$ 0.017 & -3.149  $\pm$  0.015 \\
  \hline
  \  &  \multicolumn{5}{|c|}{\textbf{Avg. Time (Secs)}} \\
  \hline
  Yacht & 30.26 & 35.18 & 26.83  & 26.69 & {\bf 24.12} \\
  \hline
  Boston & 47.14  & 61.30 & 40.53  & 40.72 & {\bf 35.93} \\
  \hline
  Concrete & 31.38  & 36.37 &  33.16  & 26.96 & {\bf 24.07} \\
  \hline
  \end{tabular}
  \end{table}
  
For both ImEQ and EVI-Im methods, the mini-batch size is set to be 100, with different batches used only for different outer iterations. 
Following the approach in Section \ref{sec:BLR}, we apply the AdaGrad algorithm with a learning rate of ${\rm lr} = 0.1$ to minimize the functional in the inner loop, and the maximum number of iterations is set to be 100 at each time step. 
Additionally, for ImEQ and AEGD methods, we set the constant $C=50$. 
The optimal learning rates are chosen as $\rm lr = 0.01$ for both ImEQ and EVI-Im methods, and 
${\rm lr}=0.001$ for AEGD, AdaGrad and SVGD methods. The learning rate of the SVGD is the same as the settings in Ref. \cite{liu2016stein}.

Fig. \ref{yacht} shows box plots of the RMSE and the predictive log-likelihood of various methods for three different datasets. 
Tab. \ref{bnnresult} shows the average RMSE, log-likelihood and CPU time over these 30 runs, along with the standard errors. For ImEQ and EVI-Im methods, we report the results after 50 outer iterations. 
For the SVGD, AdaGrad, and AEGD methods, the corresponding number of iterations is 5000.

Fig. \ref{yacht} and Table \ref{bnnresult} show that, except for AEGD, all other methods demonstrate comparable performance across the three datasets. AEGD method, however, exhibits higher RMSE and lower log-likelihood, indicating that it may not be suitable for Bayesian neural networks. Moreover, the proposed ImEQ method shows an apparent advantage in terms of running time compared to the other four methods, highlighting its efficiency. Both the RMSE and the predictive log-likelihood for the ImEQ method are comparable to those of the SVGD method, and are even better for the ``Yacht'' and ``Boston'' datasets. 
This is particularly promising, considering that the SVGD method, as reported in Ref. \cite{liu2016stein}, outperforms the probabilistic back-propagation (PBP) algorithm for Bayesian neural networks.

\section{Conclusion}

In this article, we introduce a new particle-based variational inference (ParVI) method within the Energetic Variational Inference (EVI) framework. The proposed ImEQ method is an implicit algorithm that applies energy quadratization to part of the objective function, significantly reducing computational cost compared to EVI-Im algorithm \cite{wang2021particle}. Unlike the recently developed AEGD method, which uses energy quadratization for the entire energy to derive an explicit scheme, our method remains implicit, requiring an optimization problem to be solved at each time step. While ImEQ incurs slightly higher computational cost than AEGD, it offers enhanced stability.

We evaluate the effectiveness and robustness of ImEQ method on various synthetic and real-world problems, comparing it to existing ParVI methods, including EVI-Im, AEGD, AdaGrad, and SVGD.
Numerical results show that ImEQ is more efficient than EVI-Im and often exhibits better stability than AEGD in challenging examples. It also achieves competitive performance in Bayesian logistic regression and Bayesian neural networks, compared with other ParVI methods.
The proposed algorithm also has the potential to address other optimization problems in machine learning.

\section*{Acknowledgements}
The work of L. Kang, C. Liu, and Y. Wang was partially supported by their joint National Science Foundation (USA) grant NSF DMS-2153029.
Part of this work was done when X. Bao was visiting the Department of Applied Mathematics at the Illinois Institute of Technology. 
Bao would like to acknowledge the hospitality of the Illinois Institute of Technology. 
Part of this research was performed while Y. Wang and L. Kang were visiting the Institute for Mathematical and Statistical Innovation (IMSI) at University of Chicago in March 2025, which is supported by the National Science Foundation (Grant No. DMS-1929348).

\bibliography{VI}

\end{document}